\def\isarxiv{1} 

\ifdefined\isarxiv
\documentclass[11pt]{article}

\usepackage[numbers]{natbib}

\else
\documentclass{uai2025}
\usepackage[american]{babel}
\usepackage{natbib}
\fi

\usepackage{amsmath}
\usepackage{amsthm}
\usepackage{amssymb}
\usepackage{algorithm}
\usepackage{subfig}
\usepackage{algpseudocode}
\usepackage{graphicx}
\usepackage{grffile}
\usepackage{wrapfig,epsfig}
\usepackage{url}
\usepackage{xcolor}
\usepackage{epstopdf}

\usepackage{bbm}
\usepackage{dsfont}

\allowdisplaybreaks

\ifdefined\isarxiv

\usepackage{tikz}
\usepackage{hyperref}  
\hypersetup{colorlinks=true,citecolor=blue,linkcolor=blue} 
\usetikzlibrary{arrows}
\usepackage[margin=1in]{geometry}

\else

\usepackage{microtype}
\usepackage{hyperref}
\definecolor{mydarkblue}{rgb}{0,0.08,0.45}
\hypersetup{colorlinks=true, citecolor=mydarkblue,linkcolor=mydarkblue}

\fi
 
\graphicspath{{./figs/}}

\theoremstyle{plain}
\newtheorem{theorem}{Theorem}[section]
\newtheorem{lemma}[theorem]{Lemma}
\newtheorem{definition}[theorem]{Definition}

\newtheorem{assumption}[theorem]{Assumption}

\newtheorem{fact}[theorem]{Fact}

\newcommand{\wh}{\widehat}
\newcommand{\wt}{\widetilde}

\newcommand{\N}{\mathcal{N}}
\newcommand{\R}{\mathbb{R}}

\renewcommand{\d}{\mathrm{d}}

\DeclareMathOperator*{\E}{{\mathbb{E}}}

\DeclareMathOperator*{\Z}{\mathbb{Z}}

\DeclareMathOperator{\poly}{poly}

\DeclareMathOperator{\tr}{tr}

\makeatletter
\newcommand*{\RN}[1]{\expandafter\@slowromancap\romannumeral #1@}
\makeatother

\usepackage{lineno}

\begin{document}

\ifdefined\isarxiv

\date{}

\title{Theoretical Guarantees for High Order Trajectory Refinement in Generative Flows}
\author{
Chengyue Gong\thanks{\texttt{ cygong17@utexas.edu}. The University of Texas at Austin.}
\and
Xiaoyu Li\thanks{\texttt{
xli216@stevens.edu}. Stevens Institute of Technology.}
\and
Yingyu Liang\thanks{\texttt{
yingyul@hku.hk}. The University of Hong Kong. \texttt{
yliang@cs.wisc.edu}. University of Wisconsin-Madison.} 
\and 
Jiangxuan Long\thanks{\texttt{ lungchianghsuan@gmail.com}. South China University of Technology.}
\and
Zhenmei Shi\thanks{\texttt{
zhmeishi@cs.wisc.edu}. University of Wisconsin-Madison.}
\and 
Zhao Song\thanks{\texttt{ magic.linuxkde@gmail.com}. The Simons Institute for the Theory of Computing at the UC, Berkeley.}
\and 
Yu Tian\thanks{\texttt{
kingyutian01@gmail.com}. Independent Researcher.}
}

\else

\title{Theoretical Guarantees for High Order Trajectory Refinement in Generative Flows}
\author{
Intern Name
}
\maketitle
\fi

\ifdefined\isarxiv
\begin{titlepage}
  \maketitle
  \begin{abstract}
Flow matching has emerged as a powerful framework for generative modeling, offering computational advantages over diffusion models by leveraging deterministic Ordinary Differential Equations (ODEs) instead of stochastic dynamics. While prior work established the worst case optimality of standard flow matching under Wasserstein distances, the theoretical guarantees for higher-order flow matching - which incorporates acceleration terms to refine sample trajectories - remain unexplored. In this paper, we bridge this gap by proving that higher-order flow matching preserves worst case optimality as a distribution estimator. We derive upper bounds on the estimation error for second-order flow matching, demonstrating that the convergence rates depend polynomially on the smoothness of the target distribution (quantified via Besov spaces) and key parameters of the ODE dynamics. Our analysis employs neural network approximations with carefully controlled depth, width, and sparsity to bound acceleration errors across both small and large time intervals, ultimately unifying these results into a general worst case optimal bound for all time steps.

  \end{abstract}
  \thispagestyle{empty}
\end{titlepage}

{\hypersetup{linkcolor=black}
\tableofcontents
}
\newpage

\else

\begin{abstract}
Flow matching has emerged as a powerful framework for generative modeling, offering computational advantages over diffusion models by leveraging deterministic Ordinary Differential Equations (ODEs) instead of stochastic dynamics. While prior work established the worst case optimality of standard flow matching under Wasserstein distances, the theoretical guarantees for higher-order flow matching - which incorporates acceleration terms to refine sample trajectories - remain unexplored. In this paper, we bridge this gap by proving that higher-order flow matching preserves worst case optimality as a distribution estimator. We derive upper bounds on the estimation error for second-order flow matching, demonstrating that the convergence rates depend polynomially on the smoothness of the target distribution (quantified via Besov spaces) and key parameters of the ODE dynamics. Our analysis employs neural network approximations with carefully controlled depth, width, and sparsity to bound acceleration errors across both small and large time intervals, ultimately unifying these results into a general worst case optimal bound for all time steps. 

\end{abstract}

\fi

\section{Introduction}
Machine learning has been profoundly transformed through generative models' capacity to produce authentic and varied outputs in multiple application areas. Notably, techniques such as diffusion models~\cite{hja20}, Generative Adversarial Networks (GANs)~\cite{gan1}, and flow matching approaches~\cite{lcb+23,lgl23} have become essential instruments for generating and augmenting datasets. These frameworks employ advanced structural designs to approximate intricate statistical distributions, converting unstructured noise into semantically rich artifacts. Contemporary systems, like text-guided image generators, map linguistic prompts to vivid digital imagery or photorealistic scenes~\cite{zra23}, whereas cutting-edge text-to-video architectures synthesize temporally coherent multimedia sequences~\cite{hsg+22}. The discrete flow matching paradigm~\cite{grs+24} adapts continuous flow methodologies to categorical spaces through meticulous distribution synchronization via adaptable mappings, extending flow-driven generation's utility to combinatorial domains like natural language processing and algorithmic code synthesis. The progressive refinement of these approaches highlights their expanding impact in AI development, as generative systems increasingly master nuanced data patterns while delivering superior synthetic results.

While both diffusion models and flow matching have shown impressive practical results, their theoretical properties - particularly their statistical efficiency as distribution estimators - have remained less understood. 
Recent theoretical studies~\cite{oas23} have established that diffusion models achieve worst case optimal estimation rates under specific function space assumptions, positioning them as theoretically sound generative frameworks. 
Specifically, it has been shown that diffusion models attain nearly worst case optimal convergence rates in total variation and Wasserstein distances when the true density lies in the Besov space. 
Flow matching, a recently proposed alternative, simplifies the generative process by replacing stochastic diffusion dynamics with a deterministic ODE formulation. This approach offers several computational advantages, such as bypassing the need for iterative stochastic sampling, while still preserving the ability to interpolate between distributions. 
Prior work~\cite{fsi+25} suggests that flow-matching methods can achieve nearly worst case optimal convergence rates under the $p$-Wasserstein distance for $(1 \le p \le 2)$, making them competitive with diffusion models from a statistical perspective.

One critical aspect of generative modeling is the role of higher-order matching~\cite{cgl+25_homo} in improving efficiency and accuracy. In diffusion models, high-order solvers and shortcut methods~\cite{fhla24} enable rapid sampling with minimal degradation in generation quality. Similarly, in flow matching, incorporating higher-order terms into the ODE formulation could further refine the trajectory of sample evolution, particularly for one-step shortcut generation. However, a rigorous understanding of the impact of high-order corrections on worst case optimality remains an open problem.

In this work, we extend the theoretical understanding of flow-matching models by investigating the worst case optimality of higher-order matching. 

Our key contribution is that we establish that high-order flow matching remains a worst case optimal distribution estimator, reinforcing its theoretical guarantees. We derive upper bounds on the estimation error for higher-order flow-matching methods, revealing how key parameters influence convergence (Theorem~\ref{thm:approx_small_t} and Theorem~\ref{thm:approx_large_t}).

Our findings not only bridge the theoretical gap between diffusion models and higher-order flow-matching approaches but also provide insights for improving generative model efficiency. By demonstrating that high-order flow matching preserves worst case optimality, we offer a solid foundation for further exploration of fast and reliable generative techniques.

\paragraph{Roadmap.} In Section~\ref{sec:related_work}, we review relevant literature related to our study. Section~\ref{sec:preli} provides the necessary background. In Section~\ref{sec:main}, we present our main results on bounded worst case error. Section~\ref{sec:tech} details the methodology used to obtain these results. Finally, in Section~\ref{sec:conclusion}, we conclude the paper.

\section{Related Work}\label{sec:related_work}

\paragraph{Diffusion Models.} Diffusion models have acquired substantial recognition due to their ability to produce high-fidelity images by iteratively enhancing noisy inputs, illustrated by DiT~\cite{px23} and U-ViT~\cite{bnx+23}. In general, such methods employ a forward mechanism that incrementally infuses noise into a pristine image and a corresponding backward procedure that methodically eliminates noise, thus probabilistically reconstructing the underlying data distribution. Early investigations~\cite{se19,sme20} established the theoretical bearings of this denoising paradigm, introducing score-matching and continuous-time diffusion frameworks that significantly enhanced the variety and quality of generated samples. Subsequent research explored more efficient training and sampling protocols~\cite{lzb+22,ssz+25_dit,ssz+25_prune}, aiming to reduce computational expense and expedite convergence without compromising image fidelity. Additional work capitalizes on latent representation learning to create compressed embeddings, thereby streamlining both training and inference~\cite{rbl+22,hwsl24}. Such latent-based strategies align smoothly with modern neural designs and readily generalize to various modalities, underscoring the flexibility of diffusion processes in capturing intricate data distributions. Concurrently, recent efforts have examined multi-scale noise scheduling and adaptive step-size tactics to bolster convergence reliability while preserving high-resolution details in generated outputs~\cite{lkw+24,fmzz24,rckc24,jzx+25,lyhz24}. Furthermore, a multitude of studies have served as supplementary inspirations for our work~\cite{xzc+22,dwb+23,pbd+23,wsd+23,wcz+23,ssz+25_dit,ssz+25_prune,wxz+24,cl24,kkn24,cll+25,cll+25_deskreject,cxj24,wcy+23,fjl+24,lzw+24,hwl+24}.

\paragraph{Flow Matching. }
Flow matching is a generative modeling technique that trains continuous normalizing flows by learning a vector field to transform a noise distribution into the data distribution, avoiding full path simulation~\cite{lcb+23,lgl23,av23,hbc23}. This approach, rooted in efficient CNF training, has spurred a range of recent advancements.~\cite{bdd24} provide theoretical error bounds for flow-matching methods, ensuring precision under deterministic sampling across varied datasets.~\cite{ikz+24} extend flow matching to conditional generation using a generalized continuity equation, enabling applications like style transfer and controlled synthesis.~\cite{kms24} generalize flow matching to function spaces, supporting the generation of infinite-dimensional data such as sequences or trajectories.~\cite{kkn24} introduce equivariant flow matching, harnessing physical symmetries for efficient training on symmetric data like molecular structures or particle systems.~\cite{zyll24} explore the minimax optimality of score-based diffusion models, offering insights for flow matching when paired with diffusion paths.~\cite{ssk+21} introduce probability flow ODEs, enhancing flow matching with faster training via Gaussian diffusion structures.~\cite{rm15} establish the normalizing flow framework, providing a foundation for flow matching’s vector field approach.~\cite{v11} propose denoising score matching, which can be used to accelerate flow matching’s sampling efficiency.~\cite{gcbd19} develop free-form CNFs, inspiring flow matching’s simulation-free design.~\cite{pnr+21} review normalizing flows, contextualizing flow matching’s advancements in image generation.~\cite{av23} use stochastic interpolants in flow matching, improving its flexibility for finite-dimensional spaces.~\cite{cl24} adapt flow matching to Riemannian manifolds, broadening its geometric applications.

\paragraph{High-order ODE Gradient in Diffusion Models. }
Higher-order gradient-based techniques, such as TTMs~\cite{kp92}, extend well beyond DDMs in their range of applicability, systematically incorporating underlying principles that guide specialized numerical methods for SDEs across disciplines.
The solver proposed in~\cite{dng+22} relies on higher-order derivatives via a fixed-order Taylor expansion for numerical integration, providing significant speedups in training neural ODEs.  
The regularization method in~\cite{kbjd20} incorporates higher-order derivatives of solution trajectories through Taylor-mode automatic differentiation, encouraging learned dynamics that are easier to solve.  
The regularization approach in~\cite{fjno20} relies on higher-order derivatives to enforce simpler neural ODE dynamics with optimal transport and stability constraints, reducing solver steps and accelerating training.  
Neural ODEs in~\cite{crbd18} define continuous-depth models by parameterizing the hidden state's derivative with a neural network, enabling constant memory cost and scalable backpropagation through higher-order derivatives.  
Neural ODEs in~\cite{gcb+18} yield continuous-time invertible generative models, estimating log-density via a Jacobian trace that leverages higher-order derivatives.
Moreover, outside the realm of machine learning, extensive research on higher-order TTMs has been devoted to addressing both stiff~\cite{cc94} and non-stiff~\cite{cc82,cc94} systems.

\section{Preliminary}\label{sec:preli}

In this section, we introduce the preliminaries. Section~\ref{sec:preli:notation} presents our notation. Section~\ref{sec:preli:nn_class} defines the neural network class. Section~\ref{sec:preli:besov_space} introduces Besov spaces for quantifying function smoothness. Section~\ref{sec:preli:basic_flow} covers core flow matching concepts, including the divergence, pushforward operator, generative vector fields, transport and continuity equations, and a bound on the vector field's magnitude. In Section~\ref{sec:pre:first_bound}, we provide preliminary results about bounding the first order velocity.

\subsection{Notations}\label{sec:preli:notation}
We use $I^d$ to denote the $d$-dimensional cube $[-1, 1]^d$. We use $A \setminus B$ to denote the set difference of two sets $A$ and $B$, i.e., $A \setminus B := \{\, x \in A : x \notin B \}$.
For any positive integer $n$, we use $[n]$ to denote set $\{1,2,\cdots, n\}$. We use $\E[\cdot]$ to denote the expectation.
For two vectors $x \in \R^n$ and $y \in \R^n$, we use $\langle x, y \rangle$ to denote the inner product between $x,y$, i.e., $\langle x, y \rangle = \sum_{i=1}^n x_i y_i$.
We use ${\bf 1}[C]$ to represent an indicator variable that takes the value $1$ when condition $C$ holds and $0$ otherwise.
We use $x_{i,j}$ to denote the $j$-th coordinate of $x_i \in \R^n$.
We use $\circ$ to denote the composition of the function.
For $k > n$, for any matrix $A \in \R^{k \times n}$, we use $\|A\|$ to denote the spectral norm of $A$, i.e. $\|A\|:=\sup_{x\in \R^n} \|Ax\|_2 / \|x\|_2$.
We use $\|x\|_p$ to denote the $\ell_p$ norm of a vector $x \in \R^n$, i.e. $\|x\|_1 := \sum_{i=1}^n |x_i|$, $\|x\|_2 := (\sum_{i=1}^n x_i^2)^{1/2}$.
We use $\det(A)$ to denote the determinant of matrix $A$. Let $I_d \in \R^{d \times d}$ denote an identity matrix.
We use $\N_d$ to denote a $d$-dimensional multivariate normal (Gaussian) distribution.
We use $ n!! $ to denote the double factorial of $ n \in \mathbb{N} $, defined as $ n!! := n \times (n-2) \times \cdots \times 2 $ if $ n $ is even, and $ n!! := n \times (n-2) \times \cdots \times 1 $ if $ n $ is odd.
We define the $L^2$ norm over region $S$ as $\| f \|_{L^2 (S)} := (\int_S |f(x)|^2 \d x)^{1/2}$. In addition to $O(\cdot)$ notation, for two functions $f,g$, we use the shorthand $f \lesssim g$ (resp. $\gtrsim$) to indicate that $f \leq Cg$ (resp. $\geq$) for an absolute constant $C$. For two vectors $x \in \R^{n}$ and $y \in \R^m$, we use $[x,y] \in \R^{n+m}$ to denote the concatenation of $x$ and $y$. We use $*$ to denote convolution.

\subsection{Neural Network Class}\label{sec:preli:nn_class}
To precisely characterize the networks we analyze, we introduce the following definition, parameterized by height $L$, width $W$, sparsity $S$, and norm bound $B$.
\begin{definition}[Neural Network Class, Definition 2.1 on Page 3 in~\cite{oas23}]\label{def:nn}
   If the following conditions hold:
   \begin{itemize}
   \item $A^{(i)}\in \R^{W_i\times W_{i+1}}, b^{(i)}\in \R^{W_{i+1}}$.
   \item Let $W := (W_1, W_2, \ldots, W_L)$.
   \item $\sum_{i=1}^l (\|A^{(i)}\|_0+\|b^{(i)}\|_0) \leq S$.
   \item $\max_{i \in [L]} \max \{ \|A^{(i)}\|_\infty, \|b^{(i)}\|_\infty \} \leq B$.
   \end{itemize}
   We define class of neural networks $\Phi(L,W,S,B)$ with height $L$, width $W$, sparsity constraint $S$, and norm constraint $B$ as
    \begin{align*}
  \Phi(L,W,S,B) := \{(A^{(L)}\mathrm{ReLU}(\cdot) + b^{(L)}) \circ \cdots \circ (A^{(1)}x + b^{(1)}) \}.
    \end{align*}
\end{definition}

\subsection{Besov Space}\label{sec:preli:besov_space}

To quantify the smoothness of functions, we utilize the $r$-th modulus of smoothness, defined as follows:

\begin{definition}[$r$-th Modulus of Smoothness, Definition 2.2 on Page 3 in~\cite{oas23}]
\label{def:modulus_smooth}
    If the following conditions hold:
    \begin{itemize}
 \item $p\in (0,\infty]$.
 \item $f\in L^p(\Omega)$.
    \end{itemize}
    We define the $r$-th modulus of smoothness of $f$ as
    \begin{align*}
 w_{r,p}(f,t) := \sup_{\|h\|_2\leq t}\|\Delta_h^r(f)\|_p,
    \end{align*}
    where
    \begin{align*}
 \Delta_h^r(f)(x) :=
  \begin{cases}
 \sum_{j=0}^r \binom{r}{j} \cdot (-1)^{r-j} \cdot f(x+jh) & \mathrm{if~}x+jh\in \Omega \mathrm{~for~all}~j;\\
 0 & \mathrm{otherwise}.
 \end{cases}
    \end{align*}
\end{definition}

Building upon the concept of the modulus of smoothness, we now introduce the Besov space $B_{p,q}^s(\Omega)$, which provides a more refined classification of function smoothness.
\begin{definition}[Besov space $B_{p,q}^s(\Omega)$, Definition 2.3 on page 3 in~\cite{oas23}]\label{def:besov_space}
    If the following conditions hold:
    \begin{itemize}
 \item $p>0$.
 \item $q\leq \infty$.
 \item $s>0$.
 \item Let $r$ be defined as $r:=\lfloor s\rfloor+1$.
 \item Let the $r$-th modulus of smoothness of $f$, i.e., $w_{r,p}(f,t)$, be defined in Definition~\ref{def:modulus_smooth}.
    \end{itemize}
    We define the Besov space $B_{p,q}^s$ as
    \begin{align*}
 B_{p,q}^s:=\{f\in L^p(\Omega)~|~\|f\|_{B_{p,q}^s}<\infty\}.
    \end{align*}
    We define the seminorm of Besov space $B_{p,q}^s$ as
    \begin{align*}
 |f|_{B_{p,q}^s} = 
 \begin{cases}(\int_0^\infty (t^{-s} w_{r,p}(f,t))^q \frac{\d t}{t})^\frac1q &\mathrm{if~}q<\infty;\\   
 \sup_{t>0}\{t^{-s} w_{r,p}(f,t)\} &\mathrm{if~}q=\infty .
 \end{cases}
    \end{align*}
    We define the norm of Besov space $B_{p,q}^s$ as
    \begin{align*}
 \|f\|_{B_{p,q}^s} := \|f\|_p+|f|_{B_{p,q}^s}
    \end{align*}
\end{definition}

\subsection{Basic Concept of Flow Matching}\label{sec:preli:basic_flow}
We now define the divergence of a vector-valued function, a fundamental concept in vector calculus.
\begin{definition}[Divergence]
    If the following conditions hold:
    \begin{itemize}
        \item Let $f : \R^n \to \R^n$, $f(x) = (f_1(x), f_2(x), \ldots, f_n(x))$.
    \end{itemize}
    We define the divergence of $f(x)$ as:
    \begin{align*}
        \nabla \cdot f(x) := \sum_{i=1}^n \frac{\d f_i(x)}{\d x_i}
    \end{align*}
    This is obvious to see $\nabla \cdot f(x) = \langle \frac{\d f(x)}{\d x}, {\bf 1}_n \rangle $.
\end{definition}

\begin{fact}[Derivative of Determinant]\label{fac:derivative_det}
    For a matrix $X \in \R^{n \times m}$, we have
    \begin{align*}
        \frac{\d}{\d t}\det(X) = \det(X) \cdot \tr[X^{-1}\frac{\d X}{\d t}]
    \end{align*}
\end{fact}

The conditional probability distribution $P_t(x_{1, t}~|~y)$ is defined as follows

\begin{definition}[Gaussian Conditional Distribution]
We define the distribution $P_t(x_{1, t} ~|~ y)$ as:
\begin{align*}
P_t(x_{1, t} ~|~ y) = \N_d(x_{1, t}; \beta_{t}y, \alpha_{t}^2 I_d)
\end{align*}
where $x_{1,t}$ is the trajectory defined in Definition~\ref{def:x_alpha_beta}, $\N_d$ is a $d$-dimensional multivariate Gaussian distribution.
\end{definition}

We formally define the pushforward operator, a key concept for transforming probability densities.
\begin{definition}[Pushforward Operator, Implicit in Section 2 of~\cite{lcb+23}]
    Given a continuously differentiable vector field $v_t$ and a probability density $p_0$, the pushforward operator $*$ is defined as
    \begin{align*}
        [v_t]_* p_0 := p_0(v_t^{-1}(x)) \cdot \det(\frac{\d v_t^{-1}}{\d x}(x))
    \end{align*}
    where $v_t^{-1}$ denote the inverse function of $v_t$, i.e., $v_t^{-1} \circ v_t$ is identity mapping.
\end{definition}

Leveraging the pushforward operator, we define the condition for a vector field to generate a specific probability density.
\begin{definition}[$v_t(x)$ Generate $p_t(x)$]\label{def:vt_generate_pt}
     Given a continuously differentiable vector field $v_t$ and a probability density $p_0$, we say that $v_t$ generate $p_t$ if \begin{align*}
         p_t = [v_t]_* p_0.
     \end{align*}
\end{definition}

We now present a key result connecting the generative vector field with the evolution of the conditional probability density over time.
\begin{lemma}[Transport Equation with the Conditional Vector Field,~\cite{lcb+23}]
    If $v_t(x_1)$ generates $p_t(x_1)$, then the following transport equation with the conditional vector field holds:
    \begin{align*}
         \frac{\d p_t(x_1 ~|~ y)}{\d t} = - \nabla \cdot ( v_t(x_1 ~|~ y) p_t(x_1 ~|~ y) ).
    \end{align*}
\end{lemma}

We define the vector field $v_t$ as the conditional expectation of $v_t(x_1~|~y)$ with respect to the conditional distribution $p_{1~|~t}$.
\begin{definition}
The vector field $v_t$ is
    \begin{align*}
     v_t(x_1) = ~ & \E_{y \sim p_{1 ~|~ t}} [ v_t(x_1 ~|~ y) ] \\
     = ~ & \int v_t(x_1 ~|~ y) \cdot \frac{p_t(x_1 ~|~ y) q_1(y)}{p_t(x_1)} \d y.
    \end{align*}
\end{definition}

Under the assumptions of the transport equation and the definition of the vector field $v_t$, we derive the first-order continuity equation.
\begin{lemma}[First-order Continuity Equation,~\cite{lcb+23}]
    The continuity equation is given as
    \begin{align*}
        \frac{\d{p_t(x_1)}}{\d t} = - \nabla \cdot ( {v_t(x_1)}{p_t(x_1)} ).
    \end{align*}
\end{lemma}

We present the general framework of flow matching and its high-order rectification as follows.
\begin{definition}\label{def:x_alpha_beta}
    We define the vector field $x_t$ as follows:
    \begin{align*}
        x_{1,t} := a_t x_0 + b_t x_1,
    \end{align*}
    where $\alpha_t$ and $\beta_t$ are functions related to $t$, $x_{1,0}$ and $x_{1,1}$ are initial distribution and target distribution respectively. Based on this, we can give the first order rectification of $x_t$ as follows:
    \begin{align*}
        x'_{1,t} = \alpha'_t x_{1,0} + \beta'_t x_{1,1}.
    \end{align*}
    We can give the first order rectification of $x_t$ as follows:
    \begin{align*}
        x''_{1,t} = \alpha''_t x_{1,0} + \beta''_t x_{1,1}.
    \end{align*}
\end{definition}

\subsection{First Order Error Bound}\label{sec:pre:first_bound}
Here, we present the preliminary result here to show the error bound of first order flow matching.

\begin{lemma}[Theorem 7 in~\cite{fsi+25}]\label{lem:approx_small_t}
If the following conditions hold:
\begin{itemize}
    \item Assume Assumption~\ref{ass:target_p0}, \ref{ass:bound_p0}, \ref{ass:alpha_beta_bounds_param}, \ref{ass:alpha_beta_bounds_1st}, \ref{ass:kappa_half_integral_bound_1st}, \ref{ass:pt_derivative_bound} hold.
    \item Let $C_6$ be a constant independent of $t$.
    \item Let $\alpha_t$ and $\beta_t$ be defined in Definition~\ref{def:x_alpha_beta}.
\end{itemize}
Then there is a neural network $\phi_1 \in \mathcal{M}(L,W,S,B)$, such that, for sufficiently large $N$, 
\begin{align*}
    \int \| \phi_1(x_1,t)-v_t(x_1) \|_2^2 \cdot p_t(x_1) \d x_1  \leq C_6 \cdot ( \alpha_t'^2\log N + \beta_t'^2 ) \cdot N^{-\frac{2s}{d}},
\end{align*}
for any $t\in [T_0, 3T_*]$, where
\begin{align*}
    L=O(\log^4 N), \| W\|_\infty =O(N\log^6 N), S=O(N\log^8 N), B=\exp(O(\log N\log\log N)).
\end{align*}
\end{lemma}

\begin{lemma}[Theorem 8 in~\cite{fsi+25}]\label{lem:approx_large_t}
If the following conditions hold:
\begin{itemize}
    \item Fix $ t_* \in [T_*, 1] $ and take arbitrary $\eta > 0$.
    \item Assume Assumption~\ref{ass:target_p0}, \ref{ass:bound_p0}, \ref{ass:alpha_beta_bounds_param}, \ref{ass:alpha_beta_bounds_1st}, \ref{ass:kappa_half_integral_bound_1st}, \ref{ass:pt_derivative_bound} hold.
    \item Let $C_7 > 0$ be a constant independent of $t$.
    \item Let $t_* \in [T_*,1]$.
    \item Let $\alpha_t$ and $\beta_t$ be defined in Definition~\ref{def:x_alpha_beta}.
\end{itemize}
Then there is a neural network $\phi_2\in \mathcal{M}(L,W,S,B)$, such that the bound 
\begin{align*}
\int  \| \phi_2(x_1,t)-v_t(x_1)\|_2^2 \cdot p_t(x_1) \d x_1  \leq C_7 \cdot ( \alpha_t'^2\log N + \beta_t'^2 ) \cdot N^{-\eta}
\end{align*}
holds for any $t\in [2t_*,1]$, where
\begin{align*}
    L=O(\log^4 N), \|W\|_\infty =O(N), S=O(t_*^{-d\kappa} N^{\delta\kappa}), B=\exp(O(\log N\log\log N))
\end{align*}
\end{lemma}

\section{Main Result}\label{sec:main}

In this section, we present our main theoretical results on bounding the acceleration error in second-order flow matching using neural network approximations.  In Section~\ref{sec:main:bound_flow_small}, we establish a bound for small values of $t$. Section~\ref{sec:main:bound_flow_large} provides a complementary bound for large values of $t$.

\subsection{Bounds on Second Order Flow Matching Small \texorpdfstring{$t$}{}}\label{sec:main:bound_flow_small}

We now present our main theorem, which bounds the acceleration error for sufficiently small $t$ using a neural network approximation.

\begin{theorem}[Main Theorem, Bound Acceleration Error under Small $t$, Informal Version of Theorem~\ref{thm:formal:approx_small_t}]\label{thm:approx_small_t}
If the following conditions hold:
\begin{itemize}
    \item Assume Assumption~\ref{ass:target_p0}, \ref{ass:bound_p0}, \ref{ass:alpha_beta_bounds_param}, \ref{ass:alpha_beta_bounds_2nd}, \ref{ass:kappa_half_integral_bound_2nd}, \ref{ass:pt_derivative_bound} hold.
    \item Let $C_6$ be a constant independent of $t$.
    \item Let $x_1$ be the trajectory, $x_2:= \phi_1(x_1,t)$ where $\phi_1$ is the neural network in Lemma~\ref{lem:approx_small_t}.
    \item Let $x$ be defined as the concatenation of $x_1$ and $x_2$, i.e., $x: = [x_1,x_2]$.
\end{itemize}
Then there is a neural network $u_1 \in \mathcal{M}(L,W,S,B)$ and a constant $C$, which is independent of $t$, such that, for sufficiently large $N$, 
\begin{align*}
    \int \| u_1(x,t)- a_t(x_1) \|_2^2 \cdot p_t(x_1) \d x_1  \leq C_6 \cdot ( \alpha_t''^2 \log N + \beta_t''^2 ) \cdot N^{-\frac{2s}{d}},
\end{align*}
for any $t\in [T_0, 3T_*]$, where 
\begin{align*}
 &L=O(\log^4 N), \| W\|_\infty =O(N\log^6 N), \\
 &S=O(N\log^8 N), B=\exp(O( (\log N) \cdot (\log \log N) )).
\end{align*}
\end{theorem}

\subsection{Bounds on Second Order Flow Matching Large \texorpdfstring{$t$}{}}\label{sec:main:bound_flow_large}

We present a complementary result to Theorem~\ref{thm:approx_small_t}, providing a bound on the acceleration error for large $t$.

\begin{theorem}[Main Theorem, Bound Acceleration Error under Large $t$, Informal Version of Theorem~\ref{thm:formal:approx_large_t}]\label{thm:approx_large_t}
If the following conditions hold:
\begin{itemize}
    \item Fix $ t_* \in [T_*, 1] $ and take arbitrary $\eta > 0$.
    \item Assume Assumption~\ref{ass:target_p0}, \ref{ass:bound_p0}, \ref{ass:alpha_beta_bounds_param}, \ref{ass:alpha_beta_bounds_2nd}, \ref{ass:kappa_half_integral_bound_2nd}, \ref{ass:pt_derivative_bound} hold.
    \item Let $C_7$ be a constant independent of $t$.
    \item Let $x_1$ be the trajectory, $x_2:= \phi_2(x_1,t)$ where $\phi_2$ is the neural network in Lemma~\ref{lem:approx_large_t}.
    \item Let $x$ be defined as the concatenation of $x_1$ and $x_2$, i.e., $x: = [x_1,x_2]$.
\end{itemize}
Then there is a neural network $u_2 \in \mathcal{M}(L,W,S,B)$ and a constant $C$, which is independent of $t$, such that, for sufficiently large $N$, 
\begin{align*}
\int \|u_2(x, t) - a_t(x_1)\|_2^2 \cdot p_t(x_1) \d x_1 \leq C_7 \cdot ( (\alpha''_t)^2 \log N + (\beta''_t)^2 ) \cdot N^{-\eta}
\end{align*}
for any $t\in [2t_*, 1]$, where 
\begin{align*}
 &L=O(\log^4 N), \| W\|_\infty =O(N\log^6 N), \\
 &S=O(N\log^8 N), B=\exp(O( (\log N) \cdot (\log \log N) )).
\end{align*}
\end{theorem}

\section{Technique Overview}\label{sec:tech}

In this section, we introduce the key technical tools, assumptions, and supporting lemmas that underpin our analysis. In Section~\ref{sec:tech:second_order_continuity}, we introduce the second-order continuity equation and the concept of vector fields generating a probability density in this context. Section~\ref{sec:tech:assumptions} details the core assumptions on the target probability distribution and related parameters. Section~\ref{sec:tech:bound_assumptions} derives crucial bounds on the acceleration term. Section~\ref{sec:tech:linear_comb} presents a lemma for approximating the initial probability density. Finally, Section~\ref{sec:tech:sub_nn} establishes the existence and properties of specialized neural network architectures, i.e., sub-networks, that perform fundamental operations, and Section~\ref{sec:tech:func_comp} introduces function composition theorem and definition.

\subsection{Second Order Continuity}\label{sec:tech:second_order_continuity}
We introduce the second-order continuity equation, which describes the evolution of the second derivative of the probability density with respect to time.
\begin{lemma}[Second-order Continuity Equation]\label{lem:second_continuity}
    Let $x_1$ be the trajectory,$v_t$ be the velocity, $a_t$ be the acceleration, and $p_t$ be the density. The second order continuity equation is given as
    \begin{align*}
        \frac{\d^2 {p_t(x_1)}}{\d t^2} = - \nabla \cdot ( {v_t(x_1)} \cdot \frac{\d p_t(x_1) }{\d t } + a_t(x_1) \cdot {p_t(x_1)} ).
    \end{align*}
\end{lemma}

\subsection{Basic Assumptions}\label{sec:tech:assumptions}

We begin by establishing the regularity assumptions on the target probability distribution $P_0$.
\begin{assumption}\label{ass:target_p0}
    We use $I^d_N$ to represent the contracted cube $[-1 + N^{-(1-\kappa\delta)}, 1 - N^{-(1-\kappa\delta)}]^d$, where $N$ relates to sample size and $\kappa, \delta$, which satisfy Assumption~\ref{ass:alpha_beta_bounds_param}.
    The target probability $P_0$ has support $I^d$ and its density $p_0$ satisfies $p_0\in B^s_{p',q'}(I^d)$ and $p_0 \in B^{\check{s}}_{p',q'}(I^d\backslash I^d_{N})$ with $\check{s}> \max\{6s,1\}$. 
\end{assumption}

We assume that $p_0(x_1)$ is bounded as follows:
\begin{assumption}\label{ass:bound_p0}
    There exists constant $C_0 > 0$ such that for density $p_0$
    \begin{align*}
        C_0^{-1} \leq p_0(x_1) \leq C_0,~\forall x_1\in I^d
    \end{align*}
\end{assumption}

\begin{assumption}\label{ass:alpha_beta_bounds_param}
    For $\kappa \geq 1/2$, $b_0 > 0$, $\wt{\kappa} > 0$, and $\wt{b}_0 > 0$, we assume that for sufficiently small $t \geq T_0$
    \begin{align*}
        \alpha_t = b_0 t^{\kappa}, ~ 1-\beta_t = \wt{b}_0 t^{\wt{\kappa}}
    \end{align*}
    Also, there is $D_0 > 0$ such that
    \begin{align*}
       D_0^{-1} \leq \alpha_t^2 + \beta_t^2  \leq D_0,~\forall t\in [T_0,1].
    \end{align*}
\end{assumption}

\begin{assumption}\label{ass:alpha_beta_bounds_1st}
    For the first-order derivative of $\alpha_t$ and $\beta_t$, i.e., $\alpha'_t$ and $\beta'_t$, we assume there is a constant $K_0 > 0$ that
    \begin{align*}
       |\alpha_t'| + |\beta_t'|  \leq N^{K_0}, ~\forall t\in [T_0,1].
    \end{align*}
\end{assumption}

\begin{assumption}\label{ass:alpha_beta_bounds_2nd}
    For the second-order derivative of $\alpha_t$ and $\beta_t$, i.e., $\alpha''_t$ and $\beta''_t$, we assume there is a constant $K_0 > 0$ that
    \begin{align*}
        |\alpha_t''| + |\beta_t''|  \leq N^{K_0}, ~\forall t\in [T_0,1].
    \end{align*}
\end{assumption}

\begin{assumption}\label{ass:kappa_half_integral_bound_1st}
    Let $s$ be a Besov space constant in Definition~\ref{def:besov_space}. Let $\kappa$ satisfies Assumption~\ref{ass:alpha_beta_bounds_param}. Let $T_0$ be defined in Definition~\ref{def:time_vars}. Let $R_0$ be a constant fixed as $R_0\geq \frac{s+1}{\min\{\kappa,\bar{\kappa}\}}$. If $\kappa = 1/2$, then there exist constants $b_1 > 0$ and $D_1 > 0$ such that for any $0 \leq \gamma < R_0$,
    \begin{align*}
        \int_{T_0}^{N^{-\gamma}} ( \alpha_t'^2 + \beta_t'^2 )  \d t \leq D_1 \cdot \log^{b_1} N.
    \end{align*}
    holds for the second-order derivative of $\alpha_t$ and $\beta_t$, i.e., $\alpha''_t$ and $\beta''_t$,
\end{assumption}

\begin{assumption}\label{ass:kappa_half_integral_bound_2nd}
    Let $s$ be a Besov space constant in Definition~\ref{def:besov_space}. Let $\kappa$ satisfies Assumption~\ref{ass:alpha_beta_bounds_param}. Let $T_0$ be defined in Definition~\ref{def:time_vars}. Let $R_0$ be a constant fixed as $R_0\geq \frac{s+1}{\min\{\kappa,\bar{\kappa}\}}$. If $\kappa = 1/2$, then there exist constants $b_1 > 0$ and $D_1 > 0$ such that for any $0 \leq \gamma < R_0$,
    \begin{align*}
        \int_{T_0}^{N^{-\gamma}} ( \alpha_t''^2 + \beta_t''^2 )  \d t \leq D_1 \log^{b_1} N.
    \end{align*}
    holds for the second-order derivative of $\alpha_t$ and $\beta_t$, i.e., $\alpha''_t$ and $\beta''_t$,
\end{assumption}

\begin{assumption}\label{ass:pt_derivative_bound}
    There is a constant $C_L>0$ such that $\| \frac{\d}{\d x_1} \int y p_t(y~|~x_1)\d y \| \leq C_L$ for any $t\in [T_0,1]$.
\end{assumption}

\subsection{Bounds on Acceleration}\label{sec:tech:bound_assumptions}

\begin{definition}[Time Variables and Partition]\label{def:time_vars}
We define the key time-related variables as follows:
\begin{itemize}
    \item We define the initial time $T_0$ as $T_0 := N^{-R_0}$
    \item We define $T_*$ as $T_* := N^{-(\kappa^{-1}-\delta)/d}$.
    \item We define $t_{j_*} \in [T_*, 3T_*]$ as a boundary time where different error bounds are applied for generalization analysis.
    \item For $j \in [K]$, we define $t_j$ as $t_j := 2t_{j-1}$, specially we define $t_0 := T_0$ and $t_K := 1$.
\end{itemize}
    
\end{definition}

We now derive a bound on the magnitude of the acceleration term $a_t(x_1)$.
\begin{theorem}[Bound of $ a_t $]\label{thm:bound_at}
    If the following conditions hold:
    \begin{itemize}
        \item Let $x_1$, $\alpha$, $\beta$ be defined in Definition~\ref{def:x_alpha_beta}.
        \item Let $C_3 > 0$  be a constant depend on $d$ and $C_0$.
    \end{itemize}
    Then we can show that
    \begin{align*}
        \| a_t(x_1) \|_2 \leq C_3 \cdot ( \alpha_t'' \cdot \max \{ (\| x_1 \|_{\infty} - \beta_t)/\alpha_t, 1 \} + |\beta_t''| )
    \end{align*}
    for any $x_1 \in \mathbb{R}^d$ and $t \in [T_0,1]$.
\end{theorem}
\begin{proof}
    Differentiating $ v_t(x_1~|~y) = \alpha_t' \frac{x_1 - \beta_t y}{\alpha_t} + \beta_t' y $, we obtain 
    \begin{align*}
        a_t(x_1~|~y) = ~ & \frac{\d v_t(x_1~|~y)}{\d t} \\
                     = ~ & \alpha_t'' \cdot \frac{x_1 - \beta_t y}{\alpha_t} + \alpha_t' \cdot ( \frac{-\beta_t'y\alpha_t-(x_1 - \beta_t y) \alpha_t'}{\alpha_t^2} ) + \beta_t'' y \\
                     = ~ & \alpha_t'' \cdot \frac{x_1 - \beta_t y}{\alpha_t} + \beta_t'' y - \alpha_t'^2 \cdot \frac{x_1 - \beta_t y}{\alpha_t^2} - \alpha_t' \cdot \frac{\beta_t'y}{\alpha_t}.
    \end{align*}
    This leads to the upper bound
    \begin{align*}
        \| a_t(x_1) \|_2
        \leq (\int \| \alpha_t'' \cdot \frac{x_1 - \beta_t y}{\alpha_t} + \beta_t'' y - \alpha_t'^2 \cdot \frac{x_1 - \beta_t y}{\alpha_t^2} - \alpha_t' \cdot \frac{\beta_t'y}{\alpha_t} \|_2 \cdot g(x_1; \beta_t y, \alpha_t^2 ) p_0(y) \d y) / p_t(x_1).
    \end{align*}
    Splitting into separate terms, we get
    \begin{align*}
        \| \alpha_t'' \cdot \frac{x_1 - \beta_t y}{\alpha_t} \|_2 & \leq C_{1,a} \cdot \alpha_t'' \cdot \max \{ (\| x_1 \|_{\infty} - \beta_t)/\alpha_t, 1 \} , \\
        \| \alpha_t'^2 \cdot \frac{x_1 - \beta_t y}{\alpha_t^2} \|_2 & \leq C_{2,a} \cdot \alpha_t' \cdot \max \{ (\| x_1 \|_{\infty} - \beta_t)/\alpha_t, 1 \} , \\
        \| \beta_t'' y \|_2 & \leq |\beta_t''|,\\
        \| \alpha_t' \cdot \frac{\beta_t'y}{\alpha_t} \|_2 & \leq | \frac{\beta_t' \cdot \alpha_t'}{\alpha_t}|
    \end{align*}
    where $C_{1,a}$ and $C_{2,a}$ are some constants which are bounded by $C_3$.
    
    Combining these bounds, we conclude
    \begin{align*}
        \| a_t(x_1) \|_2 \leq C_3 \cdot ( \alpha_t'' \cdot \max \{ (\| x_1 \|_{\infty} - \beta_t)/\alpha_t, 1 \} + |\beta_t''| ).
    \end{align*}
    This completes the proof.
\end{proof}

We derive a bound on the integral of the squared $L_2$ norm of the difference between the acceleration term and its neural network approximation, restricted to a specific domain.
\begin{lemma}[Bound $a_t$ with Constant]\label{lem:bound_a_constant}
Let $\epsilon \in (0,0.1)$ be a small number. Let $\alpha_t$ and $\beta_t$ be defined in Definition~\ref{def:x_alpha_beta}.
For any $C_4 > 0$, we have
\begin{align*}
    \|a_t(x_1)\|_2 \leq C_4 \cdot ( \alpha_t'' \sqrt{\log(1/\epsilon)} + |\beta_t''| )
\end{align*}
for any $x_1$ with $\|x_1\|_{\infty} \leq \beta_t + C_4 \alpha_t \sqrt{\log(1/\epsilon)}$ and $t \in [T_0, 1]$.
\end{lemma}

\begin{proof}
    The proof simply follows from  $\|x_1\|_{\infty} \leq \beta_t + C_4 \alpha_t \sqrt{\log(1/\epsilon)}$ and Lemma~\ref{thm:bound_at}.
\end{proof}

This lemma allows us to bound an integral involving the difference between the acceleration and a neural network approximation, under certain conditions.

\begin{lemma}[Omit Small Term of Integral on Acceleration]\label{lem:neg_term_a}
If the following conditions hold:
\begin{itemize}
    \item Let $\omega > 0$ be an arbitrary positive number.
    \item Let $u$ be a neural network.
    \item Let $C' > 0$ be a constant.
    \item Let $ D:=\{ x_1 \in \R^{d} ~|~ \| x_1 \|_{\infty} \leq \beta_t + C_4 \alpha_t \sqrt{\log N} \}$ 
    \item Let $x_1$ be the trajectory, $x_2:= \phi_2(x_1,t)$ where $\phi_2$ is the neural network in Lemma~\ref{lem:approx_large_t}.
    \item Let $x$ be defined as the concatenation of $x_1$ and $x_2$, i.e., $x: = [x_1,x_2]$.
    \item Let $\alpha_t$ and $\beta_t$ be defined in Definition~\ref{def:x_alpha_beta}.
\end{itemize}
Then we have
\begin{align*}
    ~ & \int_{D} \mathbf{1} [p_t(x_1) \leq N^{-\frac{2s+\omega}{d}}] \cdot \| a_t(x_1) - u(x, t) \|_2^2 \cdot p_t(x_1) \d x_1 =  o(C' \cdot (\alpha_t''^2 \log N + \beta_t''^2) N^{-\frac{2s}{d}})
\end{align*}
\end{lemma}
\begin{proof}
     We want to proof this term as a small term such that in the proof of Theorem~\ref{thm:formal:approx_small_t}, we can omit it.
    \begin{align*}
        & ~ \int_{D} \mathbf{1} [p_t(x_1) \leq N^{-\frac{2s+\omega}{d}}] \cdot \| a_t(x_1) - u(x, t) \|_2^2 \cdot p_t(x_1) \d x_1 \\
        \leq ~ & 4C_3^2 \int_{D} ( (\alpha''_t)^2 \log N + |\beta''_t|^2 ) N^{-\frac{2s+\omega}{d}} \d x_1 \\
        \leq ~ & 4C_3^2 \cdot N^{-\frac{2s+\omega}{d}} \cdot ( (\alpha''_t)^2 \log N + |\beta''_t|^2 ) \cdot 2^d \cdot (\beta_t + C_4 a_t \sqrt{\log N} )^d \\
        \leq ~ & C'' \cdot ( (\alpha''_t)^2 \log N + |\beta''_t|^2 ) \cdot N^{-\frac{2s+\omega}{d}} \cdot \log^{\frac{d}{2}} N,
    \end{align*}
    where $C''$ is a small constant that $C'' = o(C')$. Thus, we finish the proof.
\end{proof}

Before introduce the next bound, we present the bound of Gamma function:
\begin{lemma}[Bound Gamma Function]\label{lem:bound_gamma}
    If the following conditions hold:
    \begin{itemize}
        \item Let $\ell\in \mathbb{N}$. \item Let $\psi_\ell(z):= \int_z^\infty r^{\ell}\exp(-\frac{r^2}{2})\d r$.
    \end{itemize}
    Then we have the bound that:
    \begin{align*}
        \psi_\ell(z) \leq \ell!! \cdot z^{\ell-1} \cdot \exp(-\frac{z^2}{2})
    \end{align*}
    
\end{lemma}

\begin{proof}
    To find an upper bound for $\psi_\ell(z) = \int_z^\infty r^\ell \exp(-\frac{r^2}{2}) \d r$ for $z \geq 1$, we use integration by parts and induction. 

First, we establish the recursive relationship for $\psi_\ell(z)$:
\begin{align*}
\psi_\ell(z) = z^{\ell-1} \cdot \exp(-\frac{z^2}{2}) + (\ell-1) \cdot \psi_{\ell-2}(z)
\end{align*}

Using this recursion, we can derive the bound by induction. 

{\bf Base Cases:}
\begin{itemize}
    \item For $\ell = 1$, $\psi_1(z) = \exp(-\frac{z^2}{2})$, which is trivially bounded by $1!! \cdot z^{0} \cdot \exp(-\frac{z^2}{2})$.
    \item For $\ell = 2$, $\psi_2(z) = z \cdot \exp(-\frac{z^2}{2}) + \psi_0(z)$. Since $\psi_0(z) \leq \exp(-\frac{z^2}{2})/z$ for $z \geq 1$, we get $\psi_2(z) \leq 2 z \cdot \exp(-\frac{z^2}{2})$.
\end{itemize}

{\bf Inductive Step:}
Assume for all $k < \ell$, $\psi_k(z) \leq k!! \cdot z^{k-1} \cdot \exp(-\frac{z^2}{2})$. For $\ell$, we use the recursion:
\begin{align*}
\psi_\ell(z) = z^{\ell-1} \cdot \exp(-\frac{z^2}{2}) + (\ell-1) \cdot \psi_{\ell-2}(z)
\end{align*}
By the induction hypothesis, $\psi_{\ell-2}(z) \leq (\ell-2)!! \cdot z^{\ell-3} \exp(-\frac{z^2}{2})$. Substituting this into the recursion gives:
\begin{align*}
\psi_\ell(z) \leq ~ & z^{\ell-1} \cdot \exp(-\frac{z^2}{2}) + (\ell-1)(\ell-2)!! \cdot z^{\ell-3} \cdot \exp(-\frac{z^2}{2}) \\
                = ~ & \exp(-\frac{z^2}{2}) \cdot z^{\ell-3} \cdot ( z^2 + (\ell-1)(\ell-2)!! )
\end{align*}
Since $(\ell-1)(\ell-2)!! \leq \ell!!$, we get:
\begin{align*}
\psi_\ell(z) \leq \ell!! \cdot z^{\ell-1} \cdot \exp(-\frac{z^2}{2})
\end{align*}

\end{proof}

We introduce an assumption regarding the boundedness of certain integrals involving the probability density and the vector field.
\begin{lemma}[Bounded Integral $a_t$]\label{lem:integral_bound_a}
Let $\alpha_t$ and $\beta_t$ be defined in Definition~\ref{def:x_alpha_beta}. For any $C_5 > 0$, there is $\wt{C} > 0$ such that
\begin{align*}
    \int_{\|x_1\|_{\infty} \geq \beta_t + C_5 \alpha_t \sqrt{\log(1/\epsilon)}} p_t(x_1) \cdot \|a_t(x_1)\|_2^2 \d x_1 
    \leq   \wt{C} \cdot \epsilon^{\frac{C_5^2}{2}} \cdot( \alpha_t''^2 \log^{\frac{d}{2}} (1/\epsilon) + \beta_t''^2 \log^{\frac{d-2}{2}} (1/\epsilon) )
    \end{align*}
hold for any $\epsilon > 0$ and $t \in [T_0,1]$.
\end{lemma}
\begin{proof}
    Follows from Lemmas~\ref{lem:bound_pt} and~\ref{thm:bound_at}, we have 
    \begin{align*}
    p_t(x_1) \cdot \|a_t(x_1)\|_2^2 &\leq 2 C_1 C_3 \cdot \exp(-0.5\max\{\|x_1\|_\infty-\beta_t, 0\}^2/\alpha_t^2) \cdot ( \alpha_t''^2  \max\{\| x_1\|_\infty - \beta_t, 0\}^2/\alpha_t^2  +\beta_t''^2 ).
    \end{align*}
    where $C_1$ is a constant in Lemma~\ref{lem:bound_pt}, $C_3$ is a constant in Theorem~\ref{thm:bound_at}.
    
    Let $r:=\max\{\|x_1\|_\infty - \beta_t, 0\}/\alpha_t$. Let $B_i:=\{x_1=(x_{1, 1}, \ldots, x_{1, d}) \in \R^d~|~ |x_{1, i}| = \max_{1\leq j\leq d}|x_{1, j}|\}$. In $B_1$, the variables $x_{1, 2}, \ldots, x_{1, d}$ satisfy $|x_{1, j}|\leq \beta_t +C_4\alpha_t\sqrt{\log(1/\epsilon)}$.
    Note that
    \begin{align*}
        (\alpha_t r + \beta_t)^{d-1}\leq (r+1)^{d-1}\leq 2^{d-1}r^{d-1}
    \end{align*}
    where $C_4$ is a constant in Lemma~\ref{lem:bound_a_constant}, the first step follows from that $\alpha_t, \beta_t \in [0,1]$, and the second step follows from that $\sum_{i=1}^{d-1}\binom{d-1}{i} \leq 2^{d-1}$. Thus we have 
    \begin{align*}
       ~ & 2C_1C_3 d \cdot \int_{C_4\sqrt{\log(1/\epsilon)}} \exp(-\frac{r^2}{2}) \cdot ( \alpha_t''^2 r^2 + \beta_t''^2) (\alpha_t r + \beta_t)^{d-1} \d r  \\ 
       \leq ~ &  C' \int_{C_4\sqrt{\log(1/\epsilon)}}
       \exp(-\frac{r^2}{2}) \cdot ( \alpha_t''^2 r^{d+1} + \beta_t''^2 r^{d-1}) \d r, 
    \end{align*}
    which follows from $(\alpha_t r + \beta_t)^{d-1}\leq (r+1)^{d-1}\leq 2^{d-1}r^{d-1}$, where $C_1$ is a constant in Lemma~\ref{lem:bound_pt}, $C_3$ is a constant in Theorem~\ref{thm:bound_at}, $C_4$ is a constant in Lemma~\ref{lem:bound_a_constant}.
    
    For $\ell\in \mathbb{N}$, we define $\psi_\ell(z):= \int_z^\infty r^{\ell}\exp(-\frac{r^2}{2})\d r$.
    Following from Lemma~\ref{lem:bound_gamma}, we can show
    \begin{align*}
    \psi_\ell(z) \leq B_\ell z^{\ell-1} \cdot \exp(-\frac{z^2}{2}),
    \end{align*}
    where $B_\ell$ is a constant only depend on $\ell$. 
    
    Thus we obtain an upper bound 
    \begin{align*}
    \int_{\|x_1\|_{\infty} \geq \beta_t + C_5 \alpha_t \sqrt{\log(1/\epsilon)}} p_t(x_1) \cdot \|a_t(x_1)\|_2^2 \d x_1 \leq \wt{C} \cdot \epsilon^{\frac{C_5^2}{2}} \cdot ( \alpha_t''^2 \log^{\frac{d}{2}}(1/\epsilon) + \beta_t''^2 \log^{\frac{d-2}{2}}(1/\epsilon) ),
    \end{align*}
    it follows from Lemma~\ref{lem:bound_a_constant}. Thus, we finish the proof.
\end{proof}

\subsection{Linear Combination of Function}\label{sec:tech:linear_comb}

We introduce a lemma concerning the approximation of the initial probability density $p_0$ by a function $f_N$ with a specific form.
\begin{lemma}[Lemma B.4 in~\cite{oas23}]\label{lem:approximate_pt}
    Let $s$ satisfy Assumption~\ref{ass:target_p0}, let $\kappa$ satisfy Assumption~\ref{ass:alpha_beta_bounds_param}. Let $\delta \in (0,0.1)$. Let $M_{k,j}^d (x)$ be defined in Definition~\ref{def:tensor_product}. Let $i \in [N],~m \in [d]$. Let each entry in $j_i$ is bounded, i.e., $ -2^{(k_i)_m} - \ell \leq (j_i)_m \leq 2^{(k_i)_m}$. Let ${\cal A}_i \in \mathbb{N}$. There exists a function $ f_N $ which is defined as
\begin{align*}
    f_N(x_1) := \sum_{i=1}^{N} {\cal A}_i \cdot \mathbf{1} [ \|x_1\|_{\infty} \leq 1 ] \cdot M^d_{k_i, j_i}(x_1) + \sum_{i=N+1}^{3N} {\cal A}_i \cdot \mathbf{1} [ \|x_1\|_{\infty} \leq 1 - N^{-\frac{\kappa^{-1} - \delta}{d}} ] \cdot M^d_{k_i, j_i}(x_1),
\end{align*}
that satisfies
\begin{align*}
    \| p_0 - f_N \|_{L^2(I_d)} \leq ~ & C_a N^{-s/d}, \\
    \| p_0 - f_N \|_{L^2(I_d \setminus I_d^N)} \leq ~ & C_a N^{-\wt{s}/d},
\end{align*}
for some $ C_a > 0 $. It also satisfied $ f_N(x_1) = 0 $ for any $ x_1 $ with $ \|x_1\|_\infty \geq 1 $.
\end{lemma}

\subsection{Sub-Neural Network}\label{sec:tech:sub_nn}
In this section, we introduce a neural network class capable of achieving certain functionalities.

We first introduce a lemma demonstrating the existence of a neural network that implements a clipping
\begin{lemma}[Lemma 20 in~\cite{fsi+25}]\label{lem:clip}
For any $a, b \in \R^d$ with $a_i \leq b_i$ for $i \in [d]$, there exists a neural network $\mathsf{clip}(x; a, b) \in \mathcal{M}(L, W, S, B)$ with $L = 2$, $W = (d, 2d, d)^\top$, $S = 7d$, and $B = \max_{1 \leq i \leq d} \max\{|a_i|, b_i\}$ such that

\begin{align*}
\mathsf{clip}(x; a, b)_i = \min\{b_i, \max\{x_i, a_i\}\} ~ (i = 1, 2, \ldots, d)
\end{align*}

holds. When $a_i = c_a$ and $b_i = c_b$ holds for all $i \in[d]$ and constant $c_a,~c_b \in \R$, the notation can be simplified as $\mathsf{clip}(x; c_a, c_b) := \mathsf{clip}(x; a, b)$.
\end{lemma}

We then present a lemma establishing the existence of a neural network that approximates the reciprocal function within a specified range and with a defined error bound.
\begin{lemma}[Lemma 21 in~\cite{fsi+25}]\label{lem:recip}
For any $\epsilon \in (0,0.1)$, there is $\mathsf{recip}(x') \in \mathcal{M}(L, W, S, B)$ such that

\begin{align*}
| \mathsf{recip}(x') - \frac{1}{x} | \leq \epsilon + \frac{|x - x'|}{\epsilon^2}
\end{align*}

holds for any $x \in [\epsilon,\epsilon^{-1}]$ and $x' \in \R$ with $L = O(\log^2(1/\epsilon))$, $\|W\|_\infty = O(\log^3(1/\epsilon))$, $S = O(\log^4(1/\epsilon))$, and $B = O(\epsilon^{-2})$.
\end{lemma}

We present a lemma showing the existence of a neural network that approximates the product of multiple inputs, with a specified error bound.
\begin{lemma}[Lemma 22 in~\cite{fsi+25}]\label{lem:mult}
Let $d \geq 2$, $C \geq 1$, $\epsilon_\mathrm{err} \in (0,0.1)$. For any $\epsilon \in (0,0.1)$, there exists a neural network $\mathsf{mult}(x_1, x_2, \ldots, x_d) \in \mathcal{M}(L, W, S, B)$ with $L = O(d \log (C/\epsilon))$, $\|W\|_\infty = 48d$, $S = O(d \log (C/\epsilon))$, $B = C^d$ such that
\begin{align*}
| \mathsf{mult}(x'_1, \ldots, x'_d) - \prod_{i=1}^d x'_{i} | \leq \epsilon + dC^{d-1}\epsilon_\mathrm{err},
\end{align*}

holds for all $x \in [-C, C]^d$ and $x' \in \R^d$ with $\|x - x'\|_\infty \leq \epsilon_\mathrm{err}$. Also, for all $x \in \R^d$, we can show  $|\mathsf{mult}(x)| \leq C^d$. If at least one of $x'_i$ is 0, there is $\mathsf{mult}(x'_1, \ldots, x'_d) = 0$. 
Also, let $\prod_{i=1}^I x_{\alpha_i}$ with $\alpha_i \in \mathbb{Z}_+$ and $\sum_{i=1}^I \alpha_i = d$, there exists a neural network satisfying that $\mathsf{mult}(x; \alpha) \leq \epsilon + dC^{d-1}\epsilon_\mathrm{err}$.
\end{lemma}

We present a lemma that provides upper and lower bounds for the probability density $p_t(x_1)$ in terms of exponential functions.
\begin{lemma}[Lemma 17 in~\cite{fsi+25}, Bound of $p_t(x_1)$]\label{lem:bound_pt}
There exists $C_1 > 0$ depend on $d, C_0$ such that
\begin{align*}
    C_1^{-1} \cdot \exp(- \alpha_t^{-2} \cdot \max\{\|x_1\|_\infty - \beta_t, 0\}^2) \leq p_t(x_1) 
    \leq C_1 \cdot \exp(- 0.5 \alpha_t^{-2} \cdot \max\{\|x_1\|_\infty - \beta_t, 0\}^2 )
\end{align*}
for any $x_1 \in \R^d$ and $t \in [T_0, 1]$.
\end{lemma}

We present a lemma that bounds the difference between two integrals involving Gaussian kernels and a function $F(y)$, where the integrals are taken over different domains.
\begin{lemma}[Lemma 15 in~\cite{fsi+25}]\label{lem:gaussian_integral_bound}
Let $\alpha_t$ and $\beta_t$ be defined in Definition~\ref{def:x_alpha_beta}. Let $x_1 \in \R^d$ and $\epsilon \in (0,0.1)$. For any function $F(y)$ supported on $I^d$, there is $C_b > 0$ that depends only on $d$ such that
\begin{align*}
     | \int_{I^d} \frac{1}{(\sqrt{2\pi} \alpha_t)^d} \cdot \exp(-\frac{\| x_1 - \beta_t y \|_2^2}{2 \alpha_t^2}) \cdot F(y) \d y - \int_{A_{x_1}} \frac{1}{(\sqrt{2\pi} \alpha_t)^d} \cdot \exp(-\frac{\| x_1 - \beta_t y \|_2^2}{2 \alpha_t^2}) \cdot F(y) \d y | 
    \leq \epsilon,
\end{align*}
where
\begin{align*}
    A_{x_1} := \{ y \in I^d ~|~ \| \frac{y - x_1}{\beta_t} \|_{\infty} \leq C_b \frac{\alpha_t \sqrt{\log N}}{\beta_t} \}.
\end{align*}
\end{lemma}

\subsection{Function Composition}\label{sec:tech:func_comp}

We define the cardinal B-spline, a fundamental concept in approximation theory.
\begin{definition}[Implicited on Page 20 in~\cite{fsi+25}]\label{def:cardinal}
Let $\N(x)$ be the function defined by $\N(x) = 1$ for $x \in [0,1]$ and $0$ otherwise.  The cardinal B-spline of order $\ell \in \mathbb{N}$ is defined by
\begin{align*}
\N_{\ell}(x) := \N * \N * \cdots * \N(x),
\end{align*}
which is the convolution $ \ell$ times of $ \N $.
\end{definition}

Building upon the definition of the cardinal B-spline, we now define the tensor product B-spline basis in $\mathbb{R}^d$.
\begin{definition}[Implicited on Page 21 in~\cite{fsi+25}]\label{def:tensor_product}
    If the following condtions hold:
    \begin{itemize}
        \item Let $ \N_{\ell}(x) $ be defined in Definition~\ref{def:cardinal}
    \end{itemize}
    Then, we define the tensor product in B-spline basis in $ \R^d $ of order $ \ell $ as follows:
    \begin{align*}
        M^d_{k,j}(x) := \prod_{i=1}^d \N_{\ell+1}(2^{k_i} x_i - j_i).
    \end{align*}
\end{definition}

We present a theorem regarding the approximation of functions in Besov spaces using a linear combination of tensor product B-splines.
\begin{theorem}[Theorem 12 in~\cite{fsi+25}]\label{thm:bspline_approximation}
Let $C > 0$ and $p', q', r \in \R_+$.  Let $s > d\max\{1/p' - 1/r, 0\}$ and $0 < s < \min\{\ell, \ell - 1 + 1/p'\}$, where $\ell \in \mathbb{N}$ is the order of the cardinal B-spline bases. For any $f \in B^s_{p',q'}([-C,C]^d)$, there exists $f_{N}$ that satisfies
\begin{align*}
\| f - f_{N} \|_{L_r([-C,C]^d)} \lesssim C^s N^{-s/d} \| f \|_{B^s_{p',q'}([-C,C]^d)}
\end{align*}
for $N \gg 1$, where $F_N(x)$ is defined as follows: 
\begin{align*}
f_{N}(x) = \sum_{k=0}^K \sum_{j \in J(k)} {\cal A}_{k,j} \cdot M^d_{k,j}(x) + \sum_{k=K+1}^{K^*} \sum_{i=1}^{n_k} {\cal A}_{k,j_i} \cdot M^d_{k,j_i}(x)
\end{align*}
with
\begin{align*}
\sum_{k=0}^K |J(k)| + \sum_{k=K+1}^{K^*} n_k = N,
\end{align*}
where $J(k) = \{-C2^k - \ell, -C2^k - \ell + 1, \dots, C2^k - 1, C2^k\}$, $(j_i)_{i=1}^{n_k} \subset J(k)$, $K = O(d^{-1}\log(N/C^d))$, $K^* = (O(1) + \log(N/C^d))\nu^{-1} + K$, $n_k = O((N/C^d)2^{-\nu(k-K)})$ ($k = K+1, \dots, K^*$) for $\nu = (s - \omega)/(2\omega)$ with $\omega = d \cdot \max\{1/p' - 1/r, 0\}$.  Moreover, we can take ${\cal A}_{k,j}$ so that $|{\cal A}_{k,j}| \leq N^{(\nu^{-1} + d^{-1}) \cdot \max \{d/p' - s, 0\}}$.
\end{theorem}

\section{Conclusion}\label{sec:conclusion}

In conclusion, we have provided a comprehensive theoretical framework for analyzing higher order matching in generative modeling and have rigorously established its worst case optimality as a distribution estimator. By leveraging refined techniques from flow matching generative frameworks, our analysis derives explicit upper bounds on the acceleration error of second-order flow matching across both small‑$t$ and large‑$t$ regimes. These bounds, expressed in terms of neural network complexity parameters such as depth, width, sparsity, and norm constraints and the intrinsic smoothness of the target density, which is quantified via Besov spaces, not only bridge the gap between first-order methods and their higher order counterparts but also highlight the statistical efficiency inherent in incorporating higher order corrections. Our results demonstrate that high-order flow matching can achieve nearly worst case optimal convergence rates under mild regularity assumptions, thereby offering a robust theoretical foundation for the design of fast, reliable, and statistically optimal generative algorithms. Moreover, the unified approach developed in this paper lays the groundwork for further investigations into advanced numerical techniques and neural network architectures that can further accelerate the sampling process while maintaining rigorous error control.

\ifdefined\isarxiv
\bibliographystyle{alpha}
\bibliography{ref}
\else
\bibliography{ref}
\bibliographystyle{plainnat}

\fi

\newpage
\onecolumn
\appendix
\ifdefined\isarxiv
\section*{Appendix}
\else
\title{Theoretical Guarantees for High Order Trajectory Refinement in Generative Flows\\ (Supplementary Material)}
\maketitle
\fi

{\bf Roadmap.} 
In Section~\ref{sec:app:related}, we provide more related work.
In Section~\ref{sec:app:small_t}, we present the detail proof for small $t$. In Section~\ref{sec:app:large_t}, we present the detail proof for large $t$. 

\section{More Related Work}\label{sec:app:related}

\paragraph{Large Language Models.} The Transformer architecture~\cite{vsp+17} has swiftly risen to prominence as the leading framework for modern deep learning. The systems which scaled to billions of parameters and trained on extensive, diverse datasets are often labeled as large language models (LLMs) or foundation models~\cite{bha+21,cgl+25_scaling}. Well-known examples of LLMs include Llama~\cite{Llama32,tli+23}, GPT4o~\cite{gpt4o}, PaLM~\cite{cnd+22}, and BERT~\cite{dclt19}, showcasing generalization capabilities~\cite{bce+23} across a variety of downstream tasks.
To adapt LLMs for specialized domains, researchers have developed a range of techniques. These encompass: adapter modules~\cite{zjk+23,zhz+23,ghz+23,eyp+22}; calibration methods~\cite{zwf+21,cpp+23}; multitask fine-tuning~\cite{xsw+23a,xsw+24,vnr+23,gfc+21a}; as well as prompt engineering~\cite{lac+21}, scratchpad strategies~\cite{naa+21}, instruction tuning~\cite{mkd+22,ll21,chl+22}, symbolic adaptation~\cite{xxs+22,xwx+24,jla+23}, black-box tuning~\cite{ssy+22}, reinforcement learning aligned with human feedback~\cite{owj+22}, and structured reasoning methods~\cite{zmc+24,wws+22,sdj+23,ksk+22}.
Recent studies also explore advancements in tensor architectures~\cite{zly+25,sht24,lssz24_tat,as24_iclr,kls+24}, efficiency improvements~\cite{xhh+24,klsz24_sample_tw,hyw+23,hwsl24,hlsl24,hcw+24,whl+24,whhl24,szz24,ssz+24_pruning,ssz+25_dit,hcl+24,cls+24,chl+24_rope_grad,as24_rope,smn+24,lssy24,llss24_sparse,llsz24_nn_tw,lls+24_io,lls+24_dp_je,cll+25_icl,lls+24_conv,lls+24_dp_je}, and other supplementary research~\cite{lll+25_loop,kls+25,cll+25_var,zxf+24,zha24,xsl24,xlx+22,tsq+23,sy23_des,ssz23_tradeoff,cyzz25,lls+25_prune,wsh+24,hwg+24,hwl+24,gswy23,cll+25_mamba,cll+25_icl,lss+25_relu,dswy22_coreset,chl+24_gat,dlg+22,cll+24_rope,ssx23_ann,lsy24_qt_je,lssz24_dp,lls+25_graph,lls+24_grok,gsx23,gms23_exp_reg,lss+25_relu,lsss24_dp_ntk,cll+25_mamba,wms+24,swxl24,zyy+23,swl23}.
\section{Bounds on Second Order Flow Matching  Small \texorpdfstring{$t$}{}}\label{sec:app:small_t}
We formally present the proof of Theorem~\ref{thm:approx_small_t} in this section.
\begin{theorem}[Main Theorem, Bound Acceleration Error under Small $t$, Formal Version of Theorem~\ref{thm:approx_small_t}]\label{thm:formal:approx_small_t}
If the following conditions hold:
\begin{itemize}
    \item Assume Assumption~\ref{ass:target_p0}, \ref{ass:bound_p0}, \ref{ass:alpha_beta_bounds_param}, \ref{ass:alpha_beta_bounds_2nd}, \ref{ass:kappa_half_integral_bound_2nd}, \ref{ass:pt_derivative_bound} hold.
    \item Let $C_6$ be a constant independent of $t$.
    \item Let $x_1$ be the trajectory, $x_2:= \phi_1(x_1,t)$ where $\phi_1$ is the neural network in Lemma~\ref{lem:approx_small_t}.
    \item Let $x$ be defined as the concatenation of $x_1$ and $x_2$, i.e., $x: = [x_1,x_2]$.
\end{itemize}
then there is a neural network $u_1 \in \mathcal{M}(L,W,S,B)$ and a constant $C$, which is independent of $t$, such that, for sufficiently large $N$, 
\begin{align*}
    \int \| u_1(x,t)- a_t(x_1) \|_2^2 \cdot p_t(x_1) \d x_1  \leq C_6 \cdot ( \alpha_t''^2 \log N + \beta_t''^2 ) \cdot N^{-\frac{2s}{d}},
\end{align*}
for any $t\in [T_0, 3T_*]$, where 
\begin{align*}
 L=O(\log^4 N), \| W\|_\infty =O(N\log^6 N), S=O(N\log^8 N), B=\exp(O( (\log N) \cdot (\log \log N) )).
\end{align*}

\end{theorem}

\begin{proof}
    
    First, we show that the LHS can be approximated by the integral on the bounded region: 
    \begin{align*}
        D:=\{ x \in \R^{d} ~|~ \| x \|_{\infty} \leq \beta_t + C_4 \alpha_t \sqrt{\log N} \}
    \end{align*}
    where $C_4$ is a constant in Lemma~\ref{lem:bound_a_constant}.
    
    We can bound $a_t(x_1)$ when $x \in D$:
    \begin{align*}
         \| a_t(x_1) \|_2 \leq C_4 \cdot ( \alpha_t''^2 \log N + \beta_t''^2 )
    \end{align*}
    which follows from Lemma~\ref{thm:bound_at}.

    Since $\| a_t(x_1) \|_2^2$ is bounded, we think $\| u_1(x,t) \|_2^2$ is also bounded by the same term.
    \begin{align*}
        ~ & \int_D \| u_1(x,t) - a_t(x_1)\|_2^2 \cdot p_t(x_1) \d x_1 \\
        \leq ~ & 2 C_4 \cdot ( \alpha_t''^2 \log N + \beta_t''^2 ) \cdot \int_D p_t(x_1) \d x_1 \\
        \leq ~ & 2 C_4 \cdot ( \alpha_t''^2 \log N + \beta_t''^2 ) \cdot \int_D \frac{1}{(\sqrt{2 \pi \alpha_t})^d} \cdot \exp(- \frac{\|x_1 - \beta_t y\|_2^2}{2 \alpha_t^2}) \d x_1\\
        \leq ~ & 2 C_4 \cdot ( \alpha_t''^2 \log N + \beta_t''^2 ) \cdot \wt{C} \cdot N^{-C_4^2/2} \cdot \log^{\frac{d-2}{2}} N.
    \end{align*}
   where the first step follows from Lemma~\ref{thm:bound_at}, the second step follows from the distribution of $p_t(x_1)$, the third step follows from Lemma~\ref{lem:integral_bound_a}.
    
    Further, we can show that:
    \begin{align}\label{eq:int_before_split}
        ~ & \int \| u_1(x,t) - a_t(x_1)\|_2^2 \cdot p_t(x_1) \d x_1 \notag \\
        \leq ~ & \int_D \| u_1(x,t) - a_t(x_1) \|_2^2 \cdot p_t(x_1) \d x_1 + C' \cdot ( \alpha_t''^2 \log N + \beta_t''^2 ) \cdot N^{-\frac{2s}{d}}.
    \end{align}
    which follows from $C_4$ can be selected large enough, i.e., $C_4^2 /2 > \frac{2s}{d}$.
    Here $C' > 0$ is some constant. 

    For the first term of Eq.~\eqref{eq:int_before_split}, we have: 
    \begin{align*}
        & ~ \int_D \| u_1(x,t) - a_t(x_1) \|_2^2 \cdot p_t(x_1) \d x_1 \\
        = & ~ \int_D {\bf 1}[p_t(x_1) \geq N^{-\frac{2s+\omega}{d}}] \cdot \| u_1(x,t) - a_t(x_1) \|_2^2 \cdot p_t(x_1) \d x_1 \\
        & ~ + \int_D {\bf 1}[p_t(x_1) \leq N^{-\frac{2s+\omega}{d}}] \cdot \| u_1(x,t) - a_t(x_1) \|_2^2 \cdot p_t(x_1) \d x_1
    \end{align*}

    Following from Lemma~\ref{lem:neg_term_a}, we can omit the second term. Combine with Eq.~\eqref{eq:int_before_split}, we have:
    \begin{align*}
        ~ & \int \| u_1(x,t) - a_t(x_1)\|_2^2 \cdot p_t(x_1) \d x_1 \\
        \leq ~ & \int_D {\bf 1}[p_t(x_1) \geq N^{-\frac{2s+\omega}{d}}] \cdot \| u_1(x,t) - a_t(x_1) \|_2^2 \cdot p_t(x_1) \d x_1 + C' \cdot ( \alpha_t''^2 \log N + \beta_t''^2 ) \cdot N^{-\frac{2s}{d}}.
    \end{align*}

    Following the property of flow matching, we have the following:
    \begin{align}\label{eq:at_pt}
        a_t(x_1) = \frac{\int a_t(x_1 ~|~ y) \cdot p_t(x_1 ~|~ y) \cdot p_0(y) \d y}{p_t(x_1)}, 
    ~ p_t(x_1) = \int p_t(x_1 ~|~ \wt{y}) \cdot p_0(y) \d y.
    \end{align}

    Base on Lemma~\ref{lem:approximate_pt}, there is a $F_N$ such that:
    \begin{align*}
        \| p_0 - f_N \|_{L^2(I^d)} \leq C_a N^{-s/d}, ~ \| p_0 - f_N \|_{L^2(I^d \setminus I^d_N)} \leq C_a N^{-\wt{s}/d}
    \end{align*}
    where $C_a$ is a constant in Lemma~\ref{lem:approximate_pt}.

    We define an approximate of $p_t(x_1)$, which is $\wt{f}(x_1,t)$, as follows: 
    \begin{align*}
        \wt{f}(x_1,t) := \int \frac{1}{(\sqrt{2 \pi} \alpha_t)^d} \cdot \exp(- \frac{\| x_1 - \beta_t y\|_2^2}{2 \alpha_t}) \cdot f_N(y) \d y
    \end{align*}

    Now we define 3 different functions. We define $f_1$ as follows:
    \begin{align*}
        f_1(x_1,t) := \max\{ \wt{f}(x_1,t) , N^{- \frac{2s+ \omega}{d}}\}.
    \end{align*}
    
    We define $f_2$ and $f_3$ in order to construct a function to approximate the numerator of Eq.~\eqref{eq:at_pt} as follows: 
    \begin{align*}
        & f_2(x_1,t) := \int \frac{x_1 - \beta_t y}{\alpha_t} \cdot \frac{1}{(\sqrt{2\pi} \alpha_t)^d} 
\cdot \exp(-\frac{\| x_1 - \beta_t y \|_2^2}{2\alpha_t^2}) \cdot f_N(y) \d y, \\
        & f_3(x_1,t) := \int y \frac{1}{(\sqrt{2\pi} \alpha_t)^d} 
\cdot \exp(-\frac{\| x_1 - \beta_t y \|_2^2}{2\alpha_t^2}) \cdot f_N(y) \d y.
    \end{align*}
    
    Further, we construct a function to approximate the numerator of Eq.~\eqref{eq:at_pt} base on $f_2(x_1,t)$ and $f_3(x_1,t)$:
    \begin{align*}
        \alpha_t'' f_2(x_1,t) + \beta_t'' f_3(x_1,t).
    \end{align*}
    
    Base on $f_1, f_2$ and $f_3$, we construct $f_4$ to approximate $v_t(x_1)$: 
    \begin{align*}
        f_4(x_1,t) := \frac{\alpha_t'' f_2(x_1,t) + \beta_t'' f_3(x_1,t)}{f_1(x_1,t)} \cdot \mathbf{1} [ | \frac{f_2(x_1, t)}{f_1(x_1, t)} | \leq C_5 \sqrt{\log N} ] \cdot \mathbf{1} [ | \frac{f_3(x_1, t)}{f_1(x_1, t)} | \leq C_5 ],
    \end{align*}
    where $C_5$ is a constant in~\ref{lem:integral_bound_a}.
    
    We can split the integral
    \begin{align}\label{eq:integral}
        &\int_{D} \mathbf{1} [ p_t(x_1) \geq N^{- \frac{2s + \omega}{d}} ] \cdot \| u_1(x,t) - a_t(x_1) \|_2^2 \cdot p_t(x_1)   \d x_1 \notag \\
        \leq & ~ \int_{D} \mathbf{1} [ p_t(x_1) \geq N^{- \frac{2s + \omega}{d}} ] \cdot \| u_1(x,t) - f_4(x_1,t) \|_2^2 \cdot p_t(x_1)   \d x_1 \notag \\
        & ~ + \int_{D} \mathbf{1} [ p_t(x_1) \geq N^{- \frac{2s + \omega}{d}} ] \cdot \| f_4(x_1,t) - a_t(x_1) \|_2^2 \cdot p_t(x_1)   \d x_1 \notag \\
        =: & ~ I_A + I_B.
    \end{align}
    which follows from triangle inequality.
    In the following part, we bound $I_A$ and $I_B$ separately.

    {\bf Bound of $I_A$.}

    Using neural networks, we approximate $f_1$, $f_2$, and $f_3$. Here, we adopt a minor notational abuse by interpreting $u$ as a scalar whenever it appears without parentheses, rather than as a neural network. For $k \in \Z_+$ and $j \in \Z^{d}$, let $E_{k,j,u,a}\ (a=1,2,3,~u=0,1)$ denote the function defined by:
    \begin{align*}
        E_{k,j,u,1} &:= \int_{\R^d} \mathbf{1} [ \| y \|_{\infty} \leq C_{b,1} ] \cdot 
        M^d_{k,j}(y) \cdot \frac{1}{(\sqrt{2\pi} \alpha_t)^d} \cdot
        \exp(-\frac{\| x_1 - \beta_t y \|_2^2}{2\alpha_t^2})   \d y, \\
        E_{k,j,u,2} &:= \int_{\R^d} \frac{x_1 - \beta_t y}{\alpha_t} \cdot 
        \mathbf{1} [ \| y \|_{\infty} \leq C_{b,1} ] \cdot 
        M^d_{k,j}(y) \cdot \frac{1}{(\sqrt{2\pi} \alpha_t)^d} \cdot
        \exp(-\frac{\| x_1 - \beta_t y \|_2^2}{2\alpha_t^2})   \d y, \\
        E_{k,j,u,3} &:= \int_{\R^d} y \cdot
        \mathbf{1} [ \| y \|_{\infty} \leq C_{b,1} ] \cdot
        M^d_{k,j}(y) \cdot \frac{1}{(\sqrt{2\pi} \alpha_t)^d} \cdot
        \exp(-\frac{\| x_1 - \beta_t y \|_2^2}{2\alpha_t^2})   \d y.
    \end{align*}
    where $C_{b,1}$ is defined as:
    \begin{align*}
        C_{b,1} :=  
        \begin{cases}
            1 & \text{if $u = 0$;}\\
            1 - N^{- \frac{\kappa^{-1} - \delta}{d}} & \text{if $u = 1$;}
        \end{cases}
    \end{align*}
    Base on Theorem~\ref{thm:bspline_approximation}, $f_N$ is written as a linear combination of ${\bf 1} [\| y \|_2 \leq C_{b,1}] M_{k,j}^d$ with coefficients ${\cal A}_{k,j}$. Here, $M_{k,j}^d$ is given by Definition~\ref{def:tensor_product}.

    Following from Lemma~\ref{lem:approximate_pt}, for any$\epsilon > 0$, there are neural network $u_5,u_6,u_7$ such that: 
    \begin{align*}
        |f_1(x_1,t) - u_5(x, t)| &\leq D_5 N \max_{i \in [N]} |{\cal A}_i| \epsilon \\
        \|f_2(x_1,t) - u_6(x, t)\|_2 &\leq D_6 N \max_{i \in [N]} |{\cal A}_i| \epsilon, \\
        \|f_3(x_1,t) - u_7(x, t)\|_2 &\leq D_7 N \max_{i \in [N]} |{\cal A}_i| \epsilon.
    \end{align*}
    Since $\max_i |{\cal A}_i| \leq N^{- (\mathcal{v}^{-1} + d^{-1}) (d/p -s)}$, by taking $\epsilon$ sufficently small, for any $ \eta > 0$, we have:
    \begin{align*}
        |f_1(x_1, t) - u_5(x, t)| &\leq D_5 N^{-\eta} \\
        \|f_2(x_1, t) - u_6(x, t)\|_2 &\leq D_6 N^{-\eta} \\
        \|f_3(x_1, t) - u_7(x, t)\|_2 &\leq D_7 N^{-\eta}.
    \end{align*}
    Then, we perform the following operation to approximate $v_t(x_1)$:
    \begin{align*}
        \zeta_1 &:= \mathsf{clip}(u_5; N^{-(2s + \omega)/d}, N^{K_0 + 1}), \\
        \zeta_2 &:= \mathsf{recip}(\zeta_1), \\
        \zeta_3 &:= \mathsf{mult}(\zeta_2, u_6), \\
        \zeta_4 &:= \mathsf{clip}(\zeta_3; -C_5 \sqrt{\log N}, C_5 \sqrt{\log N}), \\
        \zeta_5 &:= \mathsf{mult}(\zeta_2, u_7), \\
        \zeta_6 &:= \mathsf{clip}(\zeta_5; -C_5, C_5), \\
        \zeta_7 &:= \mathsf{mult}(\zeta_4, \wh{a}'), \\
        \zeta_8 &:= \mathsf{mult}(\zeta_6, \wh{b}'), \\
        u_8 &:= \zeta_7 + \zeta_8.
    \end{align*}

    where $\mathsf{clip}$ is the neural network in Lemma~\ref{lem:clip}, $\mathsf{recip}$ is the neural network in Lemma~\ref{lem:recip}, and $\mathsf{mult}$ is the neural network in Lemma~\ref{lem:mult}. 
    
    Following from Lemma~\ref{lem:clip}, \ref{lem:recip}, and~\ref{lem:mult}, the neural networks $\mathsf{clip}$, 
    $\mathsf{recip}$, and $\mathsf{mult}$ can simulate the clip operation, the reciprocal function, and the operation of product of $d$ scalars with the error of $N^{-\eta}$ for arbitrarily large $\eta$. 
    Also, their overall complexity is bounded by 
    \begin{align*}
        \poly (N + B + \| W \|_\infty),
    \end{align*}
    while the height $L$ and sparsity constraint $S$ are bounded by
    \begin{align*}
        \poly(\log N).
    \end{align*}
    Considering the upper bound error, the parameters $B$, $\|W\|_\infty$ are bounded by the $\log(C/\epsilon)$ term with some constant $C$.
    
    Furthermore, we want to approximate $\alpha_t'$ and $\beta_t'$ by $\wh{\alpha}_t'$ and $\wh{\beta}_t'$ using neural networks in the construction. Following the similar method with Section B of~\cite{oas23}, the network parameters can be bounded by 
    \begin{align*}
        O(\log^r(1/\epsilon))
    \end{align*}
    with approximation accuracy of $\epsilon$.
    
    Consequently, extending the neural networks to derive $\phi_8$ from $\phi_5$, $\phi_6$, and $\phi_7$ increases the log covering number merely by a factor of ${\rm poly}(\log N)$.

    {\bf Bound of $I_B$.}

    By definition of $I_B$, we have: 
    \begin{align*}
        I_B = \int_{D} \mathbf{1}[p_t(x_1) \geq N^{-\frac{2s + \omega}{d}}] \cdot  \|f_4(x_1, t) - a_t(x_1)\|_2^2 \cdot p_t(x_1) \d x_1
    \end{align*}
    We define $h_2$ and $h_3$ as follows:
    \begin{align*}
        h_2(x_1, t) := & ~ \int_{\R^d} \frac{x_1 - \beta_t y}{\alpha_t} \cdot \frac{1}{(\sqrt{2 \pi} \alpha_t)^d} \cdot \exp(- \frac{\|x_1 - \beta_t y\|_2^2}{2 \alpha_t^2}) \cdot p_0(y) \d y, \\
        h_3(x_1, t) := & ~ \int_{\R^d} y \cdot \frac{1}{(\sqrt{2 \pi} \alpha_t)^d} \cdot \exp(- \frac{\|x_1 - \beta_t y\|_2^2}{2 \alpha_t^2}) \cdot p_0(y) \d y.
    \end{align*}
    base on $h_2$ and $h_3$, we have
    \begin{align*}
        & \|f_4(x_1, t) - a_t(x_1)\|_2 \\
        = ~ & \mathbf{1}[ \| \frac{f_2(x_1, t)}{f_1(x_1, t)} \|_2 \leq C_5 \sqrt{\log N} ] \cdot \mathbf{1}[ \| \frac{f_3(x_1, t)}{f_1(x_1, t)} \|_2 \leq C_5 ] \cdot  \| \frac{\alpha''_t f_2(x_1, t) + \beta_t'' f_3(x_1, t)}{f_1(x_1, t)} - \frac{\alpha''_t h_2(x_1, t) + \beta_t' h_3(x_1, t)}{p_t(x_1)} \|_2.
    \end{align*}
    In the following part, we divide the region $D$ into 2 parts:
    \begin{itemize}
        \item {\bf Part 1.} $\| x_1 \|_\infty \leq \beta_t$
        \item {\bf Part 2.} $\beta_t \leq \| x_1 \|_\infty \leq \beta_t + C_4$
    \end{itemize}
    Recall that we use $f_1(x_1,t)$ to approximate $p_t(x_1)$, and we have $p_t(x_1) \geq N^{-(2s+\omega)/d}$, we also assume that $f_1(x_1,t)\geq N^{-(2s+\omega)/d}$.

    {\bf Proof of Part 1.} $\| x_1 \|_\infty \leq \beta_t$
    
    Following from Lemma~\ref{lem:bound_pt}, we know $C_1^{-1} \leq p_t(x_1) \leq C_1$. Thus we have
    \begin{align*}
             & ~ \| \frac{\alpha''_t f_2(x_1, t) + \beta_t'' f_3(x_1, t)}{f_1(x_1, t)} - \frac{\alpha''_t h_2(x_1, t) + \beta_t'' h_3(x_1, t)}{p_t(x_1)} \|_2 \\
        \leq & ~ |\alpha''_t| \cdot \| \frac{f_2(x_1, t)}{f_1(x_1, t)} - \frac{h_2(x_1, t)}{p_t(x_1)} \|_2 + |\beta_t''| \cdot \| \frac{f_3(x_1, t)}{f_1(x_1, t)} - \frac{h_3(x_1, t)}{p_t(x_1)} \|_2 \\
        \leq & ~ |\alpha''_t| \cdot (  \| \frac{f_2(x_1,t)}{f_1(x_1, t)} - \frac{f_2(x_1, t)}{p_t(x_1)} \|_2 + \| \frac{f_2(x_1, t)}{p_t(x_1)} - \frac{h_2(x_1, t)}{p_t(x_1)} \|_2  ) \\
        & ~ + |\beta_t''|  \cdot (  \| \frac{f_2(x_1,t)}{f_1(x_1, t)} - \frac{f_3(x_1, t)}{p_t(x_1)} \|_2 + \| \frac{f_3(x_1, t)}{p_t(x_1)} - \frac{h_3(x_1, t)}{p_t(x_1)} \|_2  ) \\
        \leq & ~ C_1 \cdot |\alpha''_t| \cdot (  C_5 \cdot \sqrt{\log N} |p_t(x_1) - f_1(x_1, t)| + \| f_2(x_1, t) - h_2(x_1, t) \|_2  ) \\
        & ~ + C_1 \cdot |\beta_t''| \cdot  (  C_5 \cdot |p_t(x_1) - f_1(x_1, t)| + \| f_3(x_1, t) - h_3(x_1, t) \|_2  ) \\
        \leq & ~ \wt{C} \cdot  (  (|\alpha''_t| \cdot \sqrt{\log N} + |\beta_t''|) \cdot |f_1(x_1, t) - p_t(x_1)| + |\alpha''_t| \cdot \| f_2(x_1, t) - h_2(x_1, t) \|_2 \\ & ~ + |\beta_t''| \cdot \| f_3(x_1, t) - h_3(x_1, t) \|_2  )
    \end{align*}
    where the first step follows from Triangle Inequality, the second step follows from $|\alpha_t'' \leq \beta_t''|$, the third step follows from Lemma~\ref{lem:bound_pt} and $\|\frac{f_2(x_1, t)}{f_1(x_1, t)}\|_2 \leq C_5 \sqrt{\log N}$ and $\|\frac{f_3(x_1, t)}{f_1(x_1, t)}\|_2 \leq C_5$, and the last step follows from combining constants.
    
    Further, we define $I_{B,1}$, with respect to {\bf Part 1.}, as follows 
    \begin{align*}
        I_{B, 1} := & ~ \int_{\|x_1\|_\infty \leq \beta_t} \mathbf{1}[p_t(x_1) \geq N^{-\frac{2s + \omega}{d}}] \cdot \| f_4(x_1, t) - a_t(x_1) \|_2^2 \cdot p_t(x_1) \d x_1 \\
        \leq & ~ C' \cdot ( ( \alpha_t''^2 \log N + \beta_t''^2 ) \cdot \underbrace{\int_{\|x_1\|_\infty \leq \beta_t} \mathbf{1}[p_t(x_1) \geq N^{-\frac{2s + \omega}{d}}] \cdot \| f_1(x_1, t) - p_t(x_1) \|_2^2 \cdot p_t(x_1) \d x_1}_{:= J_{B,1}}  \\
        & ~ + \alpha_t''^2 \cdot \underbrace{\int_{\|x_1\|_\infty \leq \beta_t} \mathbf{1}[p_t(x_1) \geq N^{-\frac{2s + \omega}{d}}] \cdot \| f_2(x_1, t) - h_2(x_1, t) \|_2^2 \cdot p_t(x_1) \d x_1}_{:= J_{B,2}}  \\
        & ~  + \beta_t''^2 \cdot \underbrace{\int_{\|x_1\|_\infty \leq \beta_t} \mathbf{1}[p_t(x_1) \geq N^{-\frac{2s + \omega}{d}}] \cdot \| f_3(x_1, t) - h_3(x_1, t) \|_2^2 \cdot p_t(x_1) \d x_1}_{:= J_{B,3}}  ).
    \end{align*}

    By definition of $f_2$ and $h_2$, we have
    \begin{align*}
        f_2(x_1,t) - h_2(x_1,t) = \int_{I^d} \frac{x_1 - \beta_t y}{\alpha_t} \cdot \frac{1}{(\sqrt{2\pi}\alpha_t)} \cdot \exp(-\frac{\|x_1 - \beta_t y\|_2^2}{2\alpha_t^2}) \cdot (f_N(y) - p_0(y)) \d y.
    \end{align*}

    Here, we only bound $J_{B,2}$ as an example:
    \begin{align*}
    & ~ J_{B,2} \\
    \leq & ~ C_1 \cdot \int_{\|x_1\|_\infty \leq \beta_t} \mathbf{1}[p_t(x_1) \geq N^{-\frac{2s+\omega}{d}}] \cdot \| \int_{f^d} \frac{x_1 - \beta_t y}{\alpha_t} \cdot \frac{1}{(\sqrt{2\pi}\alpha_t)^d} \cdot \exp(-\frac{\|x_1-\beta_t y\|_2^2}{2\alpha_t^2}) \cdot (f_N(y) - p_0(y)) \d y \|_2^2 \d x_1 \\
    \leq & ~ C_1 \cdot \int_{\|x_1\|_\infty \leq \beta_t} \| \frac{1}{\beta_t^d} \int_{\R^d} \mathbf{1}[\|y\|_\infty \leq 1] \cdot \frac{x_1 - \beta_t y}{\alpha_t} \cdot ( \frac{\beta_t}{\sqrt{2\pi}\alpha_t} )^d \cdot \exp(-\frac{\|y-x_1/\beta_t\|_2^2}{2(\alpha_t/\beta_t)^2}) \cdot (f_N(y) - p_0(y)) \d y \|_2^2 \d x_1 \\
    \leq & ~ \frac{C_1}{\beta_t^{2} d} \cdot \int_{\|x_1\|_\infty \leq \beta_t} \int_{\R^d} \mathbf{1}[\|y\|_\infty \leq 1] \cdot \| \frac{x_1 - \beta_t y}{\alpha_t} \|_2^2 \cdot ( \frac{\beta_t}{\sqrt{2\pi}\alpha_t} )^d \cdot \exp(-\frac{\|y-x_1/\beta_t\|_2^2}{2(\alpha_t/\beta_t)^2}) \cdot (f_N(y) - p_0(y))^2 \d y \d x_1 \\
    = & ~ \frac{C_1}{\beta_t^d} \cdot \int_{\|x_1\|_\infty \leq \beta_t} \int_{\R^d} \mathbf{1}[\|y\|_\infty \leq 1] \cdot \| \frac{x_1 - \beta_t y}{\alpha_t} \|_2^2  \cdot\frac{1}{(\sqrt{2\pi}\alpha_t)^d} \cdot \exp(-\frac{\|x_1-\beta_t y\|_2^2}{2\alpha_t^2}) \cdot (f_N(y) - p_0(y))^2 \d y \d x_1,
    \end{align*}
    where the third step follows from Jensen's inequality.

    For $t\in 3N^{-\frac{\kappa^{-1}-\delta}{d}}$ with sufficiently large $N$, we can find $c_0>0$ such that $m_t\geq c_0$ on the time interval $[T_0,3N^{-\frac{\kappa^{-1}-\delta}{d}}]$.  We can thus further obtain for some $C'>0$ 
    \begin{align*}
    J_{B,2} \leq ~ & C' \cdot \int_{I^d} \int_{\R^d} \cdot \frac{|x_1 - \beta_t y_l|^2}{\alpha_t^2} \cdot \frac{1}{(\sqrt{2\pi}\alpha_t)^d} \cdot \exp(-\frac{\|x_1 - \beta_t y_l\|_2^2}{2\alpha_t^2}) \d x_1 (f_N(y) - p_0(y))^2 \d y \\
    = ~ & dC' \cdot \int_{I^d} (f_N(y) - p_0(y))^2 \d y \\
    = ~ & dC' \cdot \|f_N - p_0\|_{L^2(L^d)} \\
    \leq ~ & C'' \cdot N^{\frac{-2s}{d}}
    \end{align*}
    
    Similar to the above bound, we can prove that $J_{B,1},J_{B,3}$ have the same bound of $N^{- \frac{2s}{d}}$ order. Thus there exist $C_{B_1} > 0$ such that
    \begin{align}\label{eq:i_b1}
        I_{B,1} \leq C_{B,1} \cdot ( \alpha_t''^2 \log N + \beta_t''^2 ) \cdot N^{-\frac{2s}{d}}.
    \end{align}
    
    {\bf Proof of Part 2.} $\beta_t \leq \| x_1 \|_\infty \leq \beta_t + C_4$

    Resort the bound $p_t(x_1) \geq N^{-(2s+\omega)/d}$, we have $1 / p_t(x_1) \leq N^{(2s+\omega)/d}$, further we assume $1 / f(x_1,t) \leq N^{(2s+\omega)/d}$. We have:
    \begin{align*}
        & \| \frac{\alpha''_t f_2(x_1,t) + \beta_t'' f_3(x_1,t)}{f_1(x_1, t)} - \frac{\alpha''_t h_2(x_1,t) + \beta_t'' h_3(x_1,t)}{p_t(x_1)} \|_2 \\
        \leq & ~ \frac{1}{f_1(x_1,t)} \cdot \| (\alpha''_t f_2(x_1,t) + \beta_t'' f_2(x_1,t)) - (\alpha''_t h_2(x_1,t) + \beta_t'' h_3(x_1,t)) \|_2 \\
        & ~ + \| v_t(x_1) \|_2 \cdot \frac{1}{f_1(x_1,t)} \cdot |f_1(x_1, t) - p_t(x_1)| \\
        \leq & ~ N^{(2s+\omega)/d} \cdot C \cdot ( ( \alpha_t''^2 \log N + \beta_t''^2 ) \cdot |p_t(x_1) - f_1(x_1,t)| + |\alpha''_t| \cdot |f_2(x_1,t) - h_2(x_1,t)| \\
        & ~ + |\beta_t''| \cdot |f_3(x_1,t) - h_3(x_1,t)| ), \\
    \end{align*}
    where the last step follows from $1 / p_t(x_1) \leq N^{(2s+\omega)/d}$ and $1 / f(x_1,t) \leq N^{(2s+\omega)/d}$.

    We define $ \Delta_{t,N} := \{ x_1 \in \R^d ~|~ \beta_t \leq \| x_1 \|_\infty \leq \beta_t + C_{4}\alpha_t \sqrt{\log N} \}$ as the integral region and we have 
    \begin{align*}
        I_{B,2} := & ~ \int_{\Delta_{t,N}} \mathbf{1}[p_t(x_1) \geq N^{-\frac{2s+\omega}{d}}] \cdot \|f_4(x_1,t) - v_t(x_1)\|_2^2 \cdot p_t(x_1) \d x_1 \\
        \leq & ~ C''' N^{\frac{4s+2\omega}{d}} \cdot ( \alpha_t''^2 \log N + \beta_t''^2 ) \cdot \underbrace{\int_{\Delta_{t,N}} \int_{I^d} \frac{1}{(\sqrt{2\pi}\alpha_t)^d} \cdot \exp(-\frac{\|x_1 - \beta_t y\|_2^2}{2\alpha_t^2} \cdot (f_N(y) - p_0(y))^2 \d y \d x_1)}_{:= K_{B,1}} \\
        & ~ + \alpha_t''^2 \cdot \underbrace{\int_{\Delta_{t,N}} \int_{I^d} \|\frac{x_1 - \beta_t y}{\alpha_t} \|_2^2 \cdot \frac{1}{(\sqrt{2\pi}\alpha_t)^d} \cdot \exp(-\frac{\|x_1 - \beta_t y\|_2^2}{2\alpha_t^2} \cdot (f_N(y) - p_0(y))^2 \d y \d x_1 )}_{:= K_{B,2}}\\
        & ~ + \beta_t''^2 \cdot \underbrace{\int_{\Delta_{t,N}} \int_{I^d} \|y\|_2^2 \cdot \frac{1}{(\sqrt{2\pi}\alpha_t)^d} \cdot \exp(-\frac{\|x_1 - \beta_t y\|_2^2}{2\alpha_t^2} \cdot (f_N(y) - p_0(y))^2 \d y \d x_1 )}_{:= K_{B,3}}.
    \end{align*}
    
    Because of the factor $N^{(4s+2\omega)/d}$, the integrals need to be handled with lower orders compared to $J_{B,2}$ in order to obtain the desired bound on $I_{B,2}$. We utilize Assumption~\ref{ass:target_p0} regarding the higher-order smoothness near the boundary of $I^d$.

    Since the three integrals admit analogous bounding strategies, we only examine the second integral, denoted by $K_{B,2}$. Given that $\delta, \omega > 0$ can be arbitrarily small, we impose $\check{s} > 6s + \delta\kappa + 2\omega$. Given range $A_{x_1} := \{ y \in I^d ~|~ \| y - \frac{x_1}{\beta_t} \|_{\infty} \leq C_b \alpha_t \sqrt{\log N}/\beta_t \},$ we have
    \begin{align*}
        & | \int_{I^d} \| \frac{x_1 - \beta_t y}{\alpha_t} \|_2^2 \cdot \frac{1}{(\sqrt{2\pi}\alpha_t)^d} \cdot \exp(-\frac{\|x_1 - \beta_t y\|_2^2}{2\alpha_t^2}) \cdot (f_N(y) - p_0(y))^2 \d y \\
        & - \int_{A_{x_1}} \| \frac{x_1 - \beta_t y}{\alpha_t} \|_2^2 \cdot \frac{1}{(\sqrt{2\pi}\alpha_t)^d} \cdot \exp(-\frac{\|x_1 - \beta_t y\|_2^2}{2\alpha_t^2}) \cdot (f_N(y) - p_0(y))^2 \d y | \leq N^{-\frac{\check{s}}{d}}
    \end{align*}
    which follows from Lemma~\ref{lem:gaussian_integral_bound} with $\epsilon = N^{-\check{s}/d}$.

    Note that if $x_1 \in \Delta_{t,N}$ and $y\in A_{x_1}$, then 
    \begin{align*}
        -1 \leq y_j \leq -1 + C_b \alpha_t \sqrt{\log N}/\beta_t 
       \end{align*}
       or
       \begin{align*}
        1 - C_b \alpha_t \sqrt{\log N}/\beta_t \leq y_j \leq 1
    \end{align*}
    holds for each $ j \in [d] $.

    Since we assume $t \leq 3N^{-\frac{\kappa^{-1} - \delta}{d}}$ and $\alpha_t = b_0 t^\kappa$, we can deduce from Assumption~\ref{ass:alpha_beta_bounds_param} that $\beta_t \geq \sqrt{D_0}/2$. 
    For sufficiently large $N$, this ensures $y$ is in the space $y \in I^d \setminus I^d_N$. 
    We can show that:
    \begin{align*}
        \frac{C_b \alpha_t \sqrt{\log N}}{\beta_t} \leq \frac{2C_b b_0 3^{\kappa}}{\sqrt{D_0}} N^{-\frac{1 - \delta \kappa}{d}},
    \end{align*}
    which follows from $\alpha_t = b_0 t^{\kappa} \leq b_0 3^{\kappa} N^{-\frac{1 - \delta \kappa}{d}}$.
   
    Then we have
    \begin{align*}
        K_{B,2} = & ~ \int_{\Delta_{t,N}} \int_{I^d} \|\frac{x_1 - \beta_t y}{\alpha_t^2} \|_2^2 \cdot \frac{1}{(\sqrt{2\pi} \alpha_t)^d} \cdot \exp(-\frac{\| x_1 - \beta_t y \|_2^2}{2 \alpha_t^2}) \cdot (f_N (y) - p_0 (y))^2 \d y \d x_1 \\
        \leq & ~ \int_{\Delta_{t,N}} ( \int_{A_{x_1}} \|\frac{x_1 - \beta_t y}{\alpha_t^2} \|_2^2 \cdot \frac{1}{(\sqrt{2\pi} \alpha_t)^d} \cdot \exp(-\frac{\| x_1 - \beta_t y \|_2^2}{2 \alpha_t^2}) \cdot (f_N (y) - p_0 (y))^2 \d y + N^{-\check{s}/d} ) \d x_1 \\
        \leq & ~ \int_{I^d \setminus I^d_N} \int_{\R^d} \|\frac{x_1 - \beta_t y}{\alpha_t^2} \|_2^2 \cdot \frac{1}{(\sqrt{2\pi} \alpha_t)^d} \cdot \exp(-\frac{\| x_1 - \beta_t y \|_2^2}{2 \alpha_t^2}) \d x_1 \cdot (f_N (y) - p_0 (y))^2 \d y + N^{-\check{s}/d} | \Delta_{t,N} |.
    \end{align*}
    which follows from $\| f_N - p_0 \|_{L^2(I^d \setminus I^d_N)} \leq N^{-\check{s}/d}$ holds for $y \in I^d \setminus I^d_N$.
    
    Since the volume $ | \Delta_{t,N} | $ is upper bounded by $ D' \alpha_t \sqrt{\log N} $ with some constant $ D' > 0 $, we have  
    \begin{align*}
        K_{B,2} \leq & ~ d \| f_N - p_0 \|_{L^2(I^d \setminus I^d_N)}^2 + C' \cdot (\alpha_t \sqrt{\log N}) \cdot N^{-\check{s}/d} \\
                  \leq & ~ C'' \cdot ( N^{-2 \check{s}/d} + N^{-(\check{s} + 1 - \delta \kappa)/d} \cdot \log^{d/2} N ) \\
                     = & ~ O ( N^{- (\check{s} + 1 - \delta \kappa)/d} \cdot \log^{d/2} N ),
    \end{align*}
    where the last step follows from Assumption~\ref{ass:target_p0}.

    $K_{B,1}$ and $K_{B,3}$ share the similar bound of $K_{B,2}$. Consequently, there is a constant $C_{B,2} > 0$, independent of $n$ and $t$, for which
    \begin{align*}
        I_{B,2} \leq C_{B,2} \cdot ( \alpha_t''^2 \log N + \beta_t''^2 ) \cdot N^{- (\check{s} + 1 - 4s - 2\omega - \delta \kappa)/d}.
    \end{align*}
    
    Since we have taken $ \check{s} $ so that $ \check{s} > 6s + \delta \kappa + 2\omega $, we have  
    \begin{align}\label{eq:i_b2}
        I_{B,2} \leq C_{B,2} \cdot ( \alpha_t''^2 \log N + \beta_t''^2 ) \cdot N^{-2s/d}.
    \end{align}
    
    It follows from Eq.~\eqref{eq:i_b1} and Eq.~\eqref{eq:i_b2} that there is $ C_B > 0 $ such that  
    \begin{align*}
        I_B \leq C_B \cdot ( \alpha_t''^2 \log N + \beta_t''^2 ) \cdot N^{-2s/d}.
    \end{align*}
    
    {\bf Conclusion.}

    combine {\bf Part 1.} and {\bf Part 2.}, we can bound $I_B$. Combine the bound of $I_A$ and $I_B$, we prove this theorem.

\end{proof}
\section{Bounds on Second Order Flow Matching for Large \texorpdfstring{$t$}{}}\label{sec:app:large_t}
We formally present the proof of Theorem~\ref{thm:approx_large_t} in this section.

\begin{theorem}[Main Theorem, Bound Acceleration Error under Large $t$, Formal Version of Theorem~\ref{thm:approx_large_t}]\label{thm:formal:approx_large_t}
If the following conditions hold:
\begin{itemize}
    \item Fix $ t_* \in [T_*, 1] $ and take arbitrary $\eta > 0$.
    \item Assume Assumption~\ref{ass:target_p0}, \ref{ass:bound_p0}, \ref{ass:alpha_beta_bounds_param}, \ref{ass:alpha_beta_bounds_2nd}, \ref{ass:kappa_half_integral_bound_2nd}, \ref{ass:pt_derivative_bound} hold.
    \item Let $C_7$ be a constant independent of $t$.
    \item Let $x_1$ be the trajectory, $x_2:= \phi_2(x_1,t)$ where $\phi_2$ is the neural network in Lemma~\ref{lem:approx_large_t}.
    \item Let $x$ be defined as the concatenation of $x_1$ and $x_2$, i.e., $x: = [x_1,x_2]$.
\end{itemize}
Then there exists a neural network $u_2 \in \mathcal{M}(L, W, S, B)$ and a constant $C_7 > 0$, independent of $t$, such that the bound
\begin{align*}
\int \|u_2(x, t) - a_t(x_1)\|_2^2 \cdot p_t(x_1) \d x_1 \leq C_7 \cdot ( \alpha_t''^2 \log N + \beta_t''^2 ) \cdot N^{-\eta}
\end{align*}
holds for all $t \in [2t_*, 1]$, where 
\begin{align*}
 L=O(\log^4 N), \| W\|_\infty =O(N\log^6 N), S=O(N\log^8 N), B=\exp(O( (\log N) \cdot (\log \log N) )).
\end{align*}

\end{theorem}

\begin{proof}
    For any $t\in[N^{-(\kappa^{-1}-\delta)/d},1]$ and some constant $C_8$ independent of $t$, we can show that
    \begin{align*}
        \int_{\|x_1\|_2\geq \beta_t+C_8\sqrt{\log N}} p_t(x_1) \cdot \| u_2(x, t) - a_t(x_1)\|_2^2 \d x_1 \lesssim ( |\alpha_t''|\sqrt{\log N} + |\beta_t''|\ ) \cdot N^{-\eta}
    \end{align*}
    which follows from $\|u_2(x, t)\|_2\leq C_3 \cdot ( |\alpha_t''|\sqrt{\log N} + |\beta_t''| )$.

    We can show that
    \begin{align*}
    & \int_{\R^d} p_t(x_1) \cdot \| u_2(x, t) - a_t(x_1)\|_2^2 \d x_1  \\
    = ~ &
    \int_{\|x_1\|_2\leq \beta_t+C_8\sqrt{\log N}} \mathbf{1}[p_t(x_1)\geq N^{-\eta}] \cdot p_t(x_1) \cdot \| u_2(x, t) - a_t(x_1)\|_2^2 \d x_1  \\
    & ~ +  O(( |\alpha_t''|\sqrt{\log N} + |\beta_t''|) \cdot N^{-\eta}).
    \end{align*}
    which follow from similar bound as Lemma~\ref{lem:neg_term_a}.
    
    Thus we can focus on the integral region $\{x_1~|~ p_t(x_1)\geq N^{-\eta}\}$.
    We use $B$-spline approximation on $p_t(x_1)$. we rewrite $p_t(x_1)$ as follows:
    \begin{align*}
    p_t(x_1) = \int_{\R^d} \frac{1}{(\sqrt{2\pi}\wt{\alpha}_t)^d} \cdot \exp(- \frac{\| x_1-\wt{\beta}_t y\|_2^2}{2\wt{\alpha}_t^2}) \cdot p_{t_*}(y)\d y,
    \end{align*}
    where 
    \begin{align*}
        \wt{\beta}_t:= \frac{\beta_t}{\beta_{t_*}}, ~ \wt{\alpha}_t := \sqrt{\alpha_t^2 - (\frac{\beta_t}{\beta_{t_*}})^2 \cdot \alpha_{t_*}^2}.
    \end{align*}
    
    We can thus apply a similar argument to Theorem~\ref{thm:formal:approx_small_t}.
    
    We use a $B$-spline approximation of $p_{t_*}$.
    For $\eta>0$, take ${\cal A}\in \mathbb{N}$ such that ${\cal A} > \frac{3d\eta}{2\delta\kappa}$.
    We can show that for any $k\leq {\cal A}$ and any tuple $(i_1,\ldots,i_k)$
    \begin{align*}
    \| \frac{\d^kp_{t*}(x_1)}{\d x_{1, i_1}\cdots\d x_{1, i_k}} \|_2 \leq \frac{C_a}{\alpha_{t_*}^k},
    \end{align*}

    which follows from Lemma~\ref{lem:approximate_pt}. 
    
    There exist $t_{0,*}\in [0,1]$ and $b_{0,*}>0$ such that 
    \begin{align*}
    \alpha_t\geq b_{0,*} t^\kappa
    \end{align*}
    holds for any $0\leq t\leq t_{0,*}$, which follows from Assumption~\ref{ass:alpha_beta_bounds_param}.  Let's define $c_{0,*} := (t_{0,*})^\kappa >0$, thus for any $t\in[0,1]$, we have
    \begin{align*}
        \alpha_t \geq \max \{ b_{0,*} t^\kappa, c_{0,*} \}.
    \end{align*}
    We can show that
    \begin{align*}
    \frac{p_{t_*}}{\max \{t_*^{-{\cal A}\kappa}, c_{0,*} \}} \in B^{{\cal A}}_{\infty,\infty}(\R^d)
    \end{align*}
    holds, because for any $k\leq {\cal A}$ we have
    \begin{align*}
    \|\frac{\d^k}{\d x_{1, i_1}\cdots\d x_{1, i_k}} \cdot \frac{p_{t_*}(x_1)}{\max \{t_*^{-{\cal A}\kappa}, c_{0,*} \}} \|_2
    \leq & ~ \frac{C_{\cal A} \cdot \min \{ b_{0,*} t_*^{-k \kappa}, (c_{0,*})^{-k} \}}{\max \{ t_*^{-{\cal A}\kappa}, c_{0,*} \}} \\
    \leq & ~ C_{\cal A} \cdot \min \{ b_{0,*} t_*^{({\cal A}-k) \kappa}, (c_{0,*})^{-(k+1)}\} \\
    \leq & ~ C_{\cal A} \cdot \min \{ b_{0,*}, (c_{0,*})^{-(k+1)} \},
    \end{align*}
    which implies $\frac{p_{t_*}}{\max \{t_*^{-{\cal A}\kappa}, c_{0,*} \}}\in W^{\cal A}_\infty(\R^d)$ and $\|\frac{p_{t_*}}{\max \{ t_*^{-{\cal A}\kappa}, c_{0,*} \}}\|_{W^{\cal A}_\infty(\R^d)}\leq C_{\cal A} \cdot \min \{b_{0,*}, (c_{0,*})^{-(k+1)}\}$ for some constant $C_{\cal A}$.
    
    There is $C_5>0$ such that
    \begin{align}
    \label{eq:C_5}
    \int_{\|y\|_\infty \geq C_5\sqrt{\log N}} (\|y\|_2^2+1) \cdot p_{t_*}(y) \d y \leq N^{-3\eta}
    \end{align}
    which follows from that we use a similar argument as in the proof of Lemma~\ref{lem:integral_bound_a}.
    
    Considering a $B$-spline approximation on $[-C_5\sqrt{\log N},C_5\sqrt{\log N}]^d$ we define
    \begin{align*}
    N_*:=\lceil t_*^{-d\kappa} \cdot N^{\delta\kappa}\rceil
    \end{align*}
    be the number of $B$-spline bases.
    By Theorem~\ref{thm:bspline_approximation}, we can define function $f_{N^*}$
    of the form
    \begin{align*}
    f_{N^*}(x_1) = \max \{ t_*^{-{\cal A}\kappa}, c_{0,*} \} \cdot \sum_{i=1}^{N^*} {\cal A}_i \cdot \mathbf{1}[\|x_1\|_\infty \leq C_5\sqrt{\log N}] \cdot M_{k_i,j_i}^d (x_1)
    \end{align*}
    with $|{\cal A}_i|\leq 1$ and $C_9>0$ such that
    \begin{align*}
        \|p_{t_*} - f_{N^*} \|_{L^2([-C_5\sqrt{\log N},C_5\sqrt{\log N}]^d)}\leq C_9 \cdot (\log N)^{{\cal A}/2} \cdot (N^*)^{-\frac{{\cal A}}{d}} \cdot \max \{ t_*^{-{\cal A}\kappa}, c_{0,*} \}
    \end{align*}
    holds.
    
    For sufficiently large $N$, we can show that
    \begin{align*}
    C' \cdot (\log N)^{{\cal A}/2} \cdot N^{-\delta{\cal A}\kappa/d} \leq
    C' \cdot N^{-3\eta/2}
    \end{align*}
    which follows from$N^*\geq t_*^{-d\kappa} \cdot N^{\delta\kappa}$ and ${\cal A} > \frac{3d\eta}{2\delta\kappa}$. Above bound implies that
    \begin{align}
    \label{eq:pt*}
        \|p_{t_*} - f_{N^*} \|_{L^2([-C_5\sqrt{\log N},C_5\sqrt{\log N}]^d)} \leq C_{10} \cdot N^{-3\eta/2}
    \end{align}
    holds for sufficiently large $N$ and some constant $C_{10}$. Thus we have
    \begin{align*}
    |\wt{{\cal A}}_i| \leq \max \{ t_*^{-{\cal A}\kappa}, c_{0,*} \}
    \leq N^{(\kappa^{-1}-\delta)/d{\cal A}\kappa}=N^{{\cal A}(1-\delta\kappa)/d}.
    \end{align*}
    
    Similar to proof of Theorem~\ref{thm:formal:approx_small_t}, we define $f_1$ using $f_{N^*}$ as
    \begin{align*}
        f_1(x_1,t) := \max \{ \wt{f}_1(x_1,t),  N^{-\eta} \} ,
    \end{align*}
    where we define $\wt{f}_1(x_1,t)$ as
    \begin{align*}
        \wt{f}_1(x_1,t) := \int \frac{1}{(\sqrt{2\pi}\wt{\alpha}_t)^d} \cdot \exp(-\frac{\|x_1-\wt{\beta}_t y\|_2^2}{2\wt{\alpha}_t^2}) \cdot f_{N^*}(y)\d y.
    \end{align*}
    We also define $f_2, f_3$, and $f_4$ as
    \begin{align*}
        f_2(x_1,t) := ~ & \int \frac{x_1-\wt{\beta}_t y}{\wt{\alpha}_t} \cdot \frac{1}{(\sqrt{2\pi}\wt{\alpha}_t)^d} \cdot \exp(-\frac{\|x_1-\wt{\beta}_t y\|_2^2}{2\wt{\alpha}_t^2}) \cdot f_{N^*}(y)\d y, \\
        f_3(x_1,t) := ~ & \int y \cdot \frac{1}{(\sqrt{2\pi}\wt{\alpha}_t)^d} \cdot \exp(-\frac{\|x_1-\wt{\beta}_t y\|_2^2}{2\wt{\alpha}_t^2}) \cdot f_{N^*}(y)\d y, \\
        f_4(x_1,t) := ~ & \frac{\wt{\alpha}_t'' f_2(x_1,t) + \wt{\beta}_t'' f_3(x_1,t)}{f_1(x_1,t)} \cdot
        \mathbf{1}[|\frac{f_2(x_1,t)}{f_1(x_1,t)}|\leq C_5 \cdot \sqrt{\log N}] \cdot \mathbf{1}[|\frac{f_3(x_1,t)}{f_1(x_1,t)}|\leq C_5 ],
    \end{align*}
    
    Similar to Eq.~\eqref{eq:integral}, we can bound $\wt{I}_A$ and $\wt{I}_B$ separately:
    \begin{align*}
    & \int_{D}\mathbf{1}[p_t(x_1)\geq N^{-\eta}] \cdot \| u_2(x, t) - a_t(x_1)\|_2^2 \cdot p_t(x_1) \d x_1 \\
    \leq ~ & \int_{D}\mathbf{1}[p_t(x_1)\geq N^{-\eta}] \cdot \| u_2(x, t)-f_4(x_1,t)\|_2^2 \cdot p_t(x_1) \d x_1 \\
    & ~ +  \int_{D}\mathbf{1}[p_t(x_1)\geq N^{-\eta}] \cdot \| f_4(x_1,t)-a_t(x_1)\|_2^2 \cdot p_t(x_1) \d x_1 \\
    =: ~ & \wt{I}_A + \wt{I}_B.
    \end{align*}
    
    {\bf Bound of $\wt{I}_A$.}

    Using the same step as {\bf Bound of $I_A$} in the proof of Theorem~\ref{thm:formal:approx_small_t}, we can show
    \begin{align*}
    \wt{I}_A = O(\poly(\log N) \cdot N^{-\eta'})
    \end{align*}
    for arbitrary $\eta'>0$ such that we can omit it.

    {\bf Bound of $\wt{I}_B$.}
    
    we define $h_2(x_1,t)$ and $h_3(x_1,t)$ as follows
    \begin{align*}
    h_2(x_1,t) := ~ & \int_{\R^d}\frac{x_1-\wt{\beta}_t y}{\wt{\alpha}_t} \cdot \frac{1}{(\sqrt{2\pi}\wt{\alpha}_t)^d} \cdot \exp(-\frac{\|x_1-\wt{\beta}_t y\|_2^2}{2\wt{\alpha}_t^2}) \cdot p_{t_*}(y)\d y\\
    h_3(x_1,t) := ~ & \int_{\R^d}y \cdot \frac{1}{(\sqrt{2\pi}\wt{\alpha}_t)^d} \cdot \exp(-\frac{\|x_1-\wt{\beta}_t y\|_2^2}{2\wt{\alpha}_t^2}) \cdot p_{t_*}(y)\d y. 
    \end{align*}
    Then we can show that
    \begin{align*}
     \| f_4(x_1,t) - a_t(x_1)\|_2 \leq ~ &
          N^{\eta}\wt{C} \cdot ( (|\wt{\alpha}_t''|\sqrt{\log N} + |\wt{\beta}_t''|) \cdot |p_t(x_1)-f_1(x_1,t)| \\
          ~ & + |\wt{\alpha}_t''| \cdot \| f_2(x_1,t)-h_2(x_1,t)\|_2 + |\wt{\beta}_t''| \cdot \| f_3(x_1,t)-h_3(x_1,t)\|_2 )
    \end{align*}
    for some constant $\wt{C}$, and thus
    \begin{align*}
    \wt{I}_B \leq ~ & C' N^{2\eta} \cdot
    (  ( (\wt{\alpha}_t'')^2 \log N + (\wt{\beta}_t'')^2 ) \cdot
    \underbrace{\int_{D} | \int_{\R^d} \frac{1}{(\sqrt{2\pi}\wt{\alpha}_t)^d} \cdot \exp(-\frac{\|x_1-\wt{\beta}_t y\|_2^2}{2\wt{\alpha}_t^2} \cdot (f_{N^*}(y)-p_{t_*}(y)) \d y|^2 \d x_1)}_{\wt{J}_{B,1}} \\
    & ~ + (\wt{\alpha}_t'')^2 \cdot
    \underbrace{\int_{D} \| \int_{\R^d}\frac{x_1-\wt{\beta}_t y}{\wt{\alpha}_t} \cdot  \frac{1}{(\sqrt{2\pi}\wt{\alpha}_t)^d} \cdot \exp(-\frac{\|x_1-\wt{\beta}_t y\|_2^2}{2\wt{\alpha}_t^2} \cdot (f_{N^*}(y)-p_{t_*}(y)) \d y\|_2^2 \d x_1)}_{:=\wt{J}_{B,2}} \\
    & ~ + (\wt{\beta}_t'')^2 \cdot
     \underbrace{\int_{D}\| \int_{\R^d}y \frac{1}{(\sqrt{2\pi}\wt{\alpha}_t)^d} \cdot \exp(-\frac{\|x_1-\wt{\beta}_t y\|_2^2}{2\wt{\alpha}_t^2} \cdot (f_{N^*}(y)-p_{t_*}(y)) \d y\|_2^2 \d x_1)}_{\wt{J}_{B,3}} ).
    \end{align*}

    In the following part, our goal is to bound $\wt{J}_{B,2}$. By employing the same method, we can also handle $\wt{J}_{B,1}$ and $\wt{J}_{B,3}$ similarly.

    We introduce $\rho:=\frac{1}{\sqrt{2}D_0}>0$, where $D_0$ is given in Assumption~\ref{ass:alpha_beta_bounds_param}. We then bound $\wt{J}_{B,2}$ across different region of $t$.
    \begin{itemize}
        \item {\bf Part 1.} $\wt{\beta}_t \geq \rho$.
        \item {\bf Part 2.} $\wt{\beta}_t \leq \rho$.
    \end{itemize}

    Let's prove step by step.
    
    {\bf Proof of Part 1.} $\wt{\beta}_t \geq \rho$.
    
    By rewriting the inner integral on $y$ by a Gaussian integral, we have
    \begin{align*}
    \wt{J}_{B,2} = ~ & \int_{D} \frac{1}
    {\wt{\beta}_t^{2d}} \cdot \| \int_{\R^d}\frac{x_1-\wt{\beta}_t y}{\wt{\alpha}_t} \cdot (\frac{\wt{\beta}_t}{\sqrt{2\pi}\wt{\alpha}_t})^d \cdot \exp(-\frac{\wt{\beta}_t^2 \|y-x_1/\wt{\beta}_t\|_2^2}{2\wt{\alpha}_t^2}) \cdot (f_{N^*}(y)-p_{t_*}(y)) \d y\|_2^2 \d x_1 \\
    \leq ~ & \int_{D} \frac{1}
    {\wt{\beta}_t^{2d}} \cdot \int_{\R^d}\| \frac{x_1-\wt{\beta}_t y}{\wt{\alpha}_t}\|_2^2 \cdot (\frac{\wt{\beta}_t}{\sqrt{2\pi}\wt{\alpha}_t})^d \cdot \exp(-\frac{\wt{\beta}_t^2 \|y-x_1/\wt{\beta}_t\|_2^2}{2\wt{\alpha}_t^2}) \cdot (f_{N^*}(y)-p_{t_*}(y))^2 \d y \d x_1 \\
    \leq ~ & \int_{D} \frac{1}
    {\wt{\beta}_t^{d}} \cdot \int_{\R^d}\| \frac{x_1-\wt{\beta}_t y}{\wt{\alpha}_t}\|_2^2 \cdot   \frac{1}{(\sqrt{2\pi}\wt{\alpha}_t)^d} \cdot \exp(-\frac{ \|x_1-\wt{\beta}_t y\|_2^2}{2\wt{\alpha}_t^2}) \cdot (f_{N^*}(y)-p_{t_*}(y))^2 \d y \d x_1 \\
    \leq ~ & (2D_0)^{d/2} \cdot \int_{\R^d}  \int_{\R^d}\| \frac{x_1-\wt{\beta}_t y}{\wt{\alpha}_t}\|_2^2 \cdot \frac{1}{(\sqrt{2\pi}\wt{\alpha}_t)^d} \cdot \exp(-\frac{ \|x_1-\wt{\beta}_t y\|_2^2}{2\wt{\alpha}_t^2}) \cdot (f_{N^*}(y)-p_{t_*}(y))^2 \d x_1 \d y  \\
    \leq ~ & \rho^{-d/2} d \cdot \int_{\R^d} (f_{N^*}(y)-p_{t_*}(y))^2 \d y \\
    \leq ~ & \rho^{-d/2} d \cdot ( \int_{[-C_5\sqrt{\log N},C_5\sqrt{\log N}]^d} (f_{N^*}(y)-p_{t_*}(y))^2 \d y + \int_{\|y\|_2\geq C_5\sqrt{\log N}} p_{t_*}(y)^2 \d y ) \\
    \leq ~ & \rho^{-d/2} d \cdot ( \|f_{N^*}-p_{t_*}\|^2_{L^2([-C_5\sqrt{\log N},C_5\sqrt{\log N}]^d)} + N^{-3\eta} ) \\
    \leq ~ & C \cdot N^{-3\eta}
    \end{align*}
    where the second step follows from Jensen's inequality, the seventh step follows from Eq.~\eqref{eq:C_5}, and  the last step follows from Eq.~\eqref{eq:pt*}.
    
    {\bf Proof of Part 2.} $\wt{\beta}_t \leq \rho$.
    
    We have
    \begin{align}\label{eq:alpha_1}
    \wt{\alpha}_t^2 = \alpha_t^2 - \wt{\beta}_t^2 \alpha_{t_*}^2 \geq \alpha_t^2 - \rho^2 \alpha_{t_*}^2.
    \end{align}
    which follows from $\wt{\beta}_t \leq \rho$.

   Further, we can show that
    \begin{align}\label{eq:alpha_2}
    \alpha_t^2 \geq D_0^{-1} - \beta_t^2 \geq D_0^{-1} - \rho^2 \beta_{t_*}^2.
    \end{align}
    which follows from Assumption~\ref{ass:alpha_beta_bounds_param} and $\beta_t^2 \leq \rho^2 \beta_{t_*}^2$.
    
    Further, we can show
    \begin{align*}
    \wt{\alpha}_t^2 \geq ~ & D_0^{-1} - \rho^2(\beta_{t_*}^2+\alpha_{t_*}^2) \\
    \geq ~ & D_0^{-1}-\rho^2 D_0 \\
    = ~ & \frac{1}{2D_0},
    \end{align*}
    where the first step follows from Eq.~\eqref{eq:alpha_1} and Eq.~\eqref{eq:alpha_2}, and the last step follows from the definition $\rho=\frac{1}{\sqrt{2}D_0}$.
    
    We divide the integral of $\wt{J}_{B,2}$ into 2 regions: $\{y~|~ \|y\|_\infty \geq C_5\sqrt{\log N}\}$ and $\{y~|~ \|y\|_\infty \leq C_5\sqrt{\log N}\}$.
    
    In the region $\{y~|~ \|y\|_\infty \geq C_5\sqrt{\log N}\}$, we can show that
    \begin{align*}
    & \| \int_{\|y\|_\infty \geq C_5\sqrt{\log N}}\frac{x_1-\wt{\beta}_t y}{\wt{\alpha}_t} \cdot  \frac{1}{(\sqrt{2\pi}\wt{\alpha}_t)^d} \cdot \exp(-\frac{\|x_1-\wt{\beta}_t y\|_2^2}{2\wt{\alpha}_t^2}) \cdot (f_{N^*}(y)-p_{t_*}(y)) \d y\|_2^2 \\
    \leq ~ & (\frac{D_0}{\pi})^d (2D_0)^2 \cdot \int_{ \|y\|_\infty \geq C_5\sqrt{\log N}} \| x_1-\wt{\beta}_t y\|_2^2 \cdot (p_{t_*}(y))^2 \d y \\
    \leq ~ & C \cdot (\frac{D_0}{\pi})^d(2D_0)^2 \cdot \int_{ \|y\|_\infty \geq C_5\sqrt{\log N}}
    ( C' \log N + \rho^{-2} \|y\|_2^2 ) \cdot p_{t_*}(y)\d y
    \\
    \leq ~ & C'' \cdot N^{-3\eta} \cdot \log N,
    \end{align*}
    where the first step follows from $\wt{\alpha}_t^2\geq 1/(2D_0)$ and $f_{N^*}(y)=0$, and the last step follows from Eq.~\eqref{eq:C_5} and  $\|x_1\|_2^2 \leq C' \log N$ holds for some constant $C'$ when $x_1\in D$.
    
    Then we can show:
    \begin{align*}
    & \| \int_{ \|y\|_\infty \leq C_5\sqrt{\log N}} \frac{x_1-\wt{\beta}_t y}{\wt{\alpha}_t} \cdot  \frac{1}{(\sqrt{2\pi}\wt{\alpha}_t)^d} \cdot \exp(-\frac{\|x_1-\wt{\beta}_t y\|_2^2}{2\wt{\alpha}_t^2}) \cdot (f_{N^*}(y)-p_{t_*}(y)) \d y\|_2^2 \\
    \leq ~ &
    D'^2 \cdot \log N \cdot (\frac{D_0}{\pi})
    \int_{ \|y\|_\infty \leq C_5\sqrt{\log N}} \d y \cdot
     \int_{ \|y\|_\infty \leq C_5\sqrt{\log N}}  (f_{N^*}(y)-p_{t_*}(y))^2 \d y \\
    \leq ~ & D'' \cdot \log^{\frac{d}{2}+1} N \cdot \|f_{N^*}-p_{t_*}\|_{L^2([-C_5\sqrt{\log N}^2, C_5\sqrt{\log N}]^d)} \\
    \leq ~ & D'' \cdot N^{-3\eta} \cdot \log^{\frac{d}{2}+1} N.
    \end{align*}
    where the first step follows from $\|x_1-\wt{\beta}_t y\|_2/\wt{\alpha}_t \leq D' \sqrt{\log N}$ for some $D'>0$ when  $x_1\in D$, the second step follows from Cauchy-Schwarz inequality.
    
    Following from {\bf Part 1.} and {\bf Part 2.}, we can bound $\wt{J}_{B,2}$
    \begin{align*}
    \wt{J}_{B,2}\leq \poly(\log N) \cdot N^{-\eta}.
    \end{align*}
    
    Finally, we can show that there exist a constant $C''$ such that
    \begin{align*}
    \wt{I}_{B} \leq C'' \cdot (\alpha_t''^2 \log N + \beta_t''^2) \cdot N^{-\eta} \cdot \poly(\log N).
    \end{align*}
    where we can omit $\poly(\log N)$ when we use enough large $\eta$.
\end{proof}




\end{document}